\documentclass{article}

\usepackage{env}

\usepackage[utf8]{inputenc} 
\usepackage[T1]{fontenc}    
\usepackage{hyperref}
\hypersetup{
    colorlinks=True,
    linkcolor=olive,
    filecolor=green,      
    urlcolor=blue,
    citecolor=teal
}

\usepackage{url}
\usepackage{booktabs}
\usepackage{amsfonts}
\usepackage{nicefrac}
\usepackage{microtype}
\usepackage{xcolor}
\usepackage{amsmath}
\usepackage{amssymb}
\usepackage{amsthm}
\usepackage{booktabs}
\usepackage{caption}
\usepackage{graphicx}
\usepackage{subcaption}
\usepackage{algorithm}
\usepackage{algpseudocode}
\usepackage{appendix}
\usepackage{enumitem}
\usepackage{wrapfig}
\usepackage{bm}
\usepackage{pifont}
\usepackage{multirow}
\usepackage{cancel}
\usepackage{listings}

\usepackage{titletoc}
\usepackage[draft]{changes} 
\definechangesauthor[name={Flo}, color=red]{fh}
\definechangesauthor[name={Kai}, color=orange]{kl}

\newtheorem{proposition}{Proposition} 
\newtheorem{theorem}{Theorem} 
\newtheorem{lemma}{Lemma}

\newtheorem{assumption}{Assumption}

\newcommand{\vbf}{\mathbf{v}}
\newcommand{\bfa}{\mathbf{a}}
\newcommand{\abf}{\mathbf{a}}
\newcommand{\dd}{\mathbf{d}}

\newcommand{\g}{\mathbf{g}}

\newcommand{\m}{\mathbf{m}}
\newcommand{\p}{\mathbf{p}}

\newcommand{\eps}{\bm{\epsilon}}

\newcommand{\ww}{\mathbf{w}}

\newcommand{\rbf}{\mathbf{r}}

\newcommand{\M}{\mathbf{M}}

\newcommand{\X}{\mathbf{X}}

\newcommand{\A}{\mathbf{A}}

\newcommand{\C}{\mathbf{C}}

\newcommand{\B}{\mathbf{B}}

\newcommand{\W}{\mathbf{W}}
\newcommand{\D}{\mathbf{D}}
\newcommand{\G}{\mathbf{G}}
\newcommand{\PP}{\mathbf{P}}

\newcommand{\E}{\mathbb{E}}
\newcommand{\Rbf}{\mathbf{R}}

\newcommand{\Obf}{\mathbf{O}}

\newcommand{\mubf}{\boldsymbol{\mu}}

\newcommand{\sgn}{\text{sgn}}

\newcommand{\diag}{\mathrm{diag}}

\newcommand{\inner}[2]{\left \langle #1, #2 \right \rangle}
\newcommand{\norm}[1]{\left \| #1 \right\|}
\newcommand{\normF}[1]{\norm{#1}_{\mathrm{F}}}

\newcommand{\opnorm}[1]{\norm{#1}_{S_\infty}}

\newcommand{\RowNorm}[1]{\norm{#1}_{\mathrm{row}}}
\DeclareMathOperator{\Diag}{Diag}
\newcommand{\pare}[1]{\left( #1 \right)}
\newcommand{\innerF}[2]{\left\langle #1, #2 \right\rangle_{\mathrm{F}}}
\newcommand{\rowPos}{\mathbb{R}^{m \times n}_{\text{row} > 0}}

\newcommand{\Ft}{\widetilde {\mathcal{L}}}

\newcommand{\gt}{\g_t}
\newcommand{\gtp}{\g_{t+1}}
\newcommand{\mt}{\m_t}
\newcommand{\mtm}{\m_{t-1}}
\newcommand{\Rt}{\Rbf_t}
\newcommand{\Rtp}{\Rbf_{t+1}}
\newcommand{\Wt}{\W_t}
\newcommand{\Wtp}{\W_{t+1}}
\newcommand{\Mt}{\M_t}
\newcommand{\Mtm}{\M_{t-1}}
\newcommand{\Ot}{\Obf_t}
\newcommand{\gradgt}{\nabla_{\g} \Ft(\gt, \Rt)}
\newcommand{\gradRt}{\nabla_{\Rbf} \Ft(\gt, \Rt)}
\newcommand{\trnorm}[1]{\norm{#1}_{S_1}}

\newcommand{\muon}{Muon}

\newcommand{\signsgd}{SignSGD}
\newcommand{\sigg}{\sigma_\g}
\newcommand{\sigR}{\sigma_\Rbf}

\newcommand{\algname}{Muown}

\title{Muown: Row-Norm Control for Muon Optimization}

\author{%
  Kai Lion\\
  Department of Computer Science\\
  ETH Zurich, Switzerland\\
  \texttt{kai.lion@inf.ethz.ch}
  \And
  Florian Hübler\\
  Department of Computer Science\\
  ETH Zurich, Switzerland\\
  \texttt{florian.huebler@inf.ethz.ch} 
  \And
  \\
  Bingcong Li\\
  Department of Computer Science\\
  ETH Zurich, Switzerland\\
  \texttt{bingcong.li@inf.ethz.ch} 
  \And
  \\
  Antonio Orvieto\\
  ELLIS Institute Tübingen, MPI-IS\\
  Tübingen AI Center, Germany\\
  \texttt{antonio@tue.ellis.eu}
  \And
  \\
  Niao He\\
  Department of Computer Science\\
  ETH Zurich, Switzerland\\
  \texttt{niao.he@inf.ethz.ch} \\
}

\newcommand{\affmark}[1]{\textsuperscript{\normalfont #1}}

\author{%
\begin{minipage}{0.94\textwidth}
\centering
{\large
Kai Lion\affmark{1}\quad
Florian H\"ubler\affmark{1,2}\quad
Bingcong Li\affmark{1}\\[-0.1ex]
Antonio Orvieto\affmark{3}\quad
Niao He\affmark{1}
}\\[0.7em]
{\small
\affmark{1} Department of Computer Science, ETH Zurich, Switzerland\\[-0.05ex]
\affmark{2} Department of Mathematics, School of Computation, Information and Technology;\\[-0.05ex]
Technical University of Munich, Germany\\[-0.05ex]
\affmark{3} ELLIS Institute Tübingen, MPI-IS, Tübingen AI Center, Germany\\[0.6em]
\texttt{\{kai.lion,florian.huebler,bingcong.li,niao.he\}@inf.ethz.ch}\\[-0.05ex]
\texttt{antonio@tue.ellis.eu}
}
\end{minipage}%
}

\begin{document}

\maketitle

\begin{abstract}
Muon has emerged as a strong competitor to AdamW for language model pre-training, yet its behavior at scale is sensitive to weight decay. Recent work has observed that, for Muon without decoupled weight decay, the spectral norm of weight matrices drifts upward over training. Through a decomposition of the spectral norm into a row-magnitude factor and a row-coherence factor, we identify the former as the empirical driver of this drift under Muon, while the latter remains well-behaved along the trajectory. Motivated by this diagnosis, we introduce \algname{}, a drop-in replacement for Muon that treats the row-magnitude vector as an explicit optimizer variable, updating it under the $\ell_\infty$ geometry induced by the decomposition, while applying Muon unchanged to the remaining direction component. We prove that \algname{} attains the optimal non-convex rates in both deterministic and stochastic regimes under a dual norm aligned with the underlying geometries and with a stochastic noise coefficient that empirically remains below that of Muon throughout training. Across GPT-style pre-training on FineWeb-Edu with model sizes from 124M up to 2.7B parameters, \algname{} improves perplexity over Muon, SOAP, AdamW, and Lion. It also widens the plateau of near-optimal learning rates across model scales, reduces sensitivity to weight decay, and avoids the spectral norm drift at negligible step-time overhead when appropriately sharded. 
\end{abstract}

\section{Introduction}

Moving beyond the diagonal conditioning inherent to Adam and its variants \citep{kingma_adam_2015, loshchilov_decoupled_2019}, an emerging class of optimizers is engineered to respect the matrix-shaped structure of the linear layers of the Transformer architecture. Notable examples include Shampoo \citep{gupta_shampoo_2018}, SOAP \citep{vyas_soap_2024}, Scion~\citep{pethick_training_2025} and Muon \citep{jordan2024muon}. In particular, Muon has been motivated through the perspective of steepest spectral descent \citep{carlson_stochastic_2015}, which computes the update direction as the solution to the optimization problem $\arg \min_{\Delta \W \in \mathbb{R}^{m \times n}} \inner{\G}{\Delta \W} \text{ s.t. } \opnorm{\Delta \W} \leq 1$ where $\G$ is the matrix-shaped gradient of the loss with respect to $\W \in \mathbb{R}^{m \times n}$ and $\opnorm{\cdot}$ is the spectral (Schatten-$\infty$) norm. 

Despite the encouraging empirical success of Muon, it has been observed that, at scale, its formulation can suffer from numerical overflows in bfloat16 \citep{liu_muon_2025} and exploding attention logits \citep{team_kimi_2025}. At the same time, the spectral norm of weight matrices has been documented to grow consistently throughout training under Muon without weight decay \citep[Fig.\ 9]{pethick_training_2025}, violating the \emph{spectral condition} for stable feature learning \citep{yang_spectral_2024}, which prescribes $\opnorm{\W} = \Theta(\sqrt{m/n})$ and $\opnorm{\Delta\W} = \Theta(\sqrt{m/n})$. This drift away from the stable feature-learning regime could potentially be linked to the pathologies mentioned above. 

A standard remedy is decoupled weight decay \citep{chen_muon_2025, chen_lion_2025, pethick_training_2025}, which keeps $\opnorm{\W_t}$ in check by shrinking the weight matrix uniformly at every step. This intervention is non-targeted, scaling the entire spectrum uniformly, irrespective of which directions are responsible for spectral-norm growth, and has been observed to slow convergence in the early phase of training \cite[Fig.~2]{liu_muon_2025}, \cite[Fig.~8]{pethick_training_2025}. We introduce a complementary perspective: rather than constraining the spectrum as a whole, we identify systematic row-norm growth as the empirical mechanism driving the spectral-norm drift (Fig.~\ref{fig:row-scales-mlp},~\ref{fig:row-scales-wout}) and control it by reparameterizing the weight so that the row scales become explicitly trainable. To remain agnostic to the model architecture, this reparameterization is realized implicitly inside the optimizer, leaving the forward pass unchanged. The resulting procedure achieves consistent gains in perplexity across all tested model scales and learning rates, at a minimal step-time overhead compared to Muon.

\subsection{Contribution}

\begin{itemize}
    \item We decompose the spectral norm into a row-magnitude factor and a row-coherence factor (Proposition~\ref{prop:spectral-norm-decomposition}), which singles out the row magnitude as the component whose evolution tracks the spectral norm under Muon. Building on this perspective, we propose \algname{}, a drop-in replacement for Muon that promotes the row magnitudes to an explicit optimizer variable, while Muon is applied unchanged to the directional component.
    \item \algname{} retains the optimal non-convex rates of Muon in both deterministic and stochastic settings (Theorems~\ref{thm:stoch_convergence}, \ref{thm:det_convergence}), but established in the mixed dual norm matching the two update geometries. The result supports the empirical improvements of \algname{} through an improved noise term.
    \item On FineWeb-Edu pre-training from 124M to 2.7B parameters, \algname{} consistently improves perplexity over Muon, SOAP, AdamW, and Lion, maintaining a consistent perplexity advantage over the best-tuned Muon baseline. Further, it broadens the plateau of near-optimal learning rates across model widths (Fig.\ \ref{fig:width-ablation}), and alleviates the need for weight-decay tuning (Fig.\ \ref{fig:weight-decay-ablation}). The method adds only minor step-time overhead once sharded (Table~\ref{tab:timing-overhead}), matching Muon's runtime in practical scenarios.
\end{itemize}

The remainder of this work is structured as follows. In Section~\ref{sec:related-work}, we review relevant background. We then introduce \algname{} in Section~\ref{sec:method} and motivate it from the perspective of approximate spectral norm control via trainable row-scales. In Section~\ref{sec:conv_analysis}, we establish convergence of \algname{} and conclude by comparing its effectiveness against Muon, SOAP, and AdamW in Section~\ref{sec:experiments}. 

\begin{figure}[t]
    \centering
    \begin{subfigure}[t]{0.245\textwidth}
        \centering
        \includegraphics[width=\textwidth]{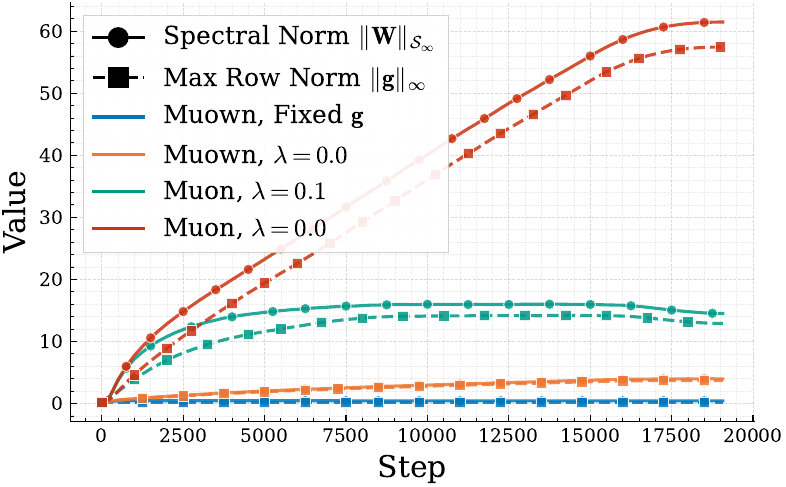}
        \caption{MLP Layer}
        \label{fig:row-scales-mlp}
    \end{subfigure}\hfill
    \begin{subfigure}[t]{0.245\textwidth}
        \centering
        \includegraphics[width=\textwidth]{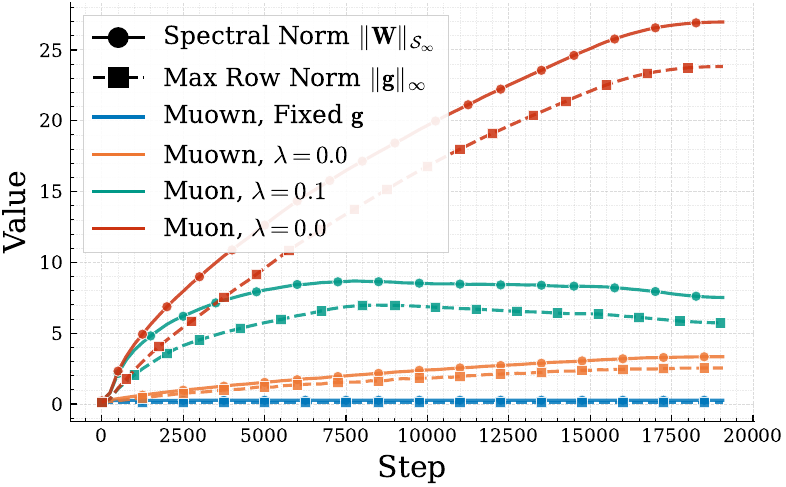}
        \caption{Output Projection}
        \label{fig:row-scales-wout}
    \end{subfigure}\hfill
    \begin{subfigure}[t]{0.245\textwidth}
        \centering
        \includegraphics[width=\textwidth]{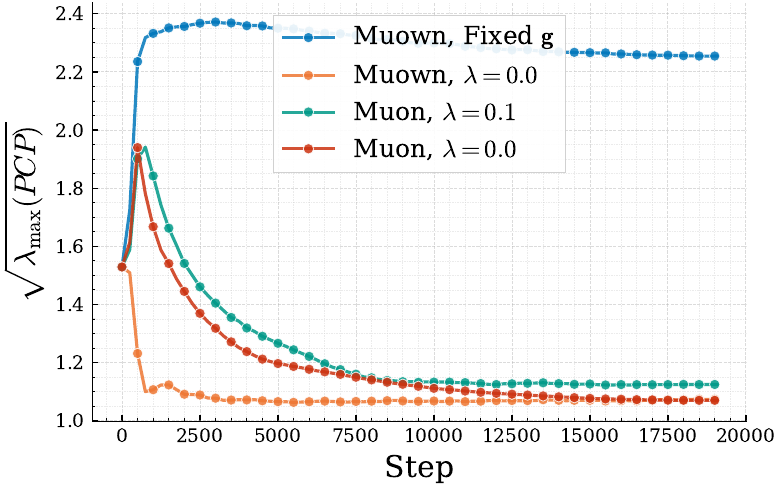}
        \caption{MLP Layer}
        \label{fig:lambda-max-fc2}
    \end{subfigure}
    \begin{subfigure}[t]{0.245\textwidth}
        \centering
        \includegraphics[width=\textwidth]{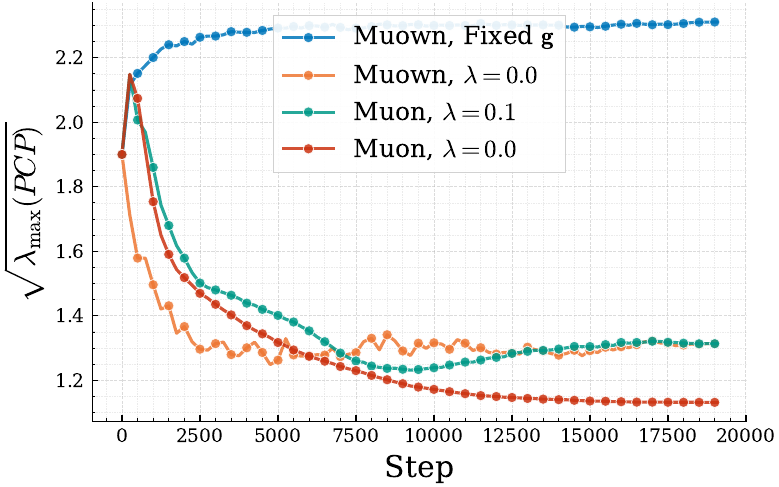}
        \caption{Output Projection}
        \label{fig:lambda-max-w_out}
    \end{subfigure}\hfill

    \caption{\textit{Left two:} Drift of the maximum row norm $\|\g_t\|_{\infty}$ drives the growth of the spectral norm $\opnorm{\W_t}$ for an MLP and attention output projection layer in a 500M transformer. When intervening by changing the parameterization s.t.\ the row magnitudes are fixed to the value at initialization, the systematic spectral norm increase vanishes. \textit{Right two:} The remaining coherence factor $\lambda_{\max}\!\bigl(\PP_t\,\C_t\,\PP_t\bigr)$ in the spectral norm exhibits no drift. $\lambda$ denotes the weight decay, for more details, see Appendix~\ref{appdx:spectral-experiments}}
    \label{fig:sv_histograms}
    \vspace{-10pt}
\end{figure}

\section{Related Work}
\label{sec:related-work}

\textbf{Shaping Gradient Spectra.} Pre-conditioning is a central tool in optimization for machine learning~\citep{bottou2018optimization}: while early approaches focus on efficient use of the Hessian or the Fisher Information~\citep{martens2015optimizing},  contemporary methods directly operate on gradients, taking their matrix structure into account. An early notable work in gradient-based structure-aware optimizers is Shampoo~\citep{gupta_shampoo_2018}. Rather than applying pre-conditioning on flattened, structureless parameter vectors, Shampoo employs two matrix-shaped pre-conditioners. An idealized variant of Shampoo can be cast as performing steepest spectral descent under the spectral norm, effectively setting the singular values of the parameter update matrix to unity  \citep{bernstein_old_2024}. SOAP \citep{vyas_soap_2024} optimizes singular values adaptively with Adam. Most prominently, Muon \citep{jordan2024muon, tuddenham_orthogonalising_2022} shapes the spectrum through orthogonalization of the momentum state, approximating the normalized steepest descent direction \citep{boyd_convex_2004} under a spectral norm constraint \citep{carlson_stochastic_2015}. 
\begin{wrapfigure}{r}{0.5\textwidth}
        \centering
        \vspace{-0pt}
         \includegraphics[width=0.5\textwidth]{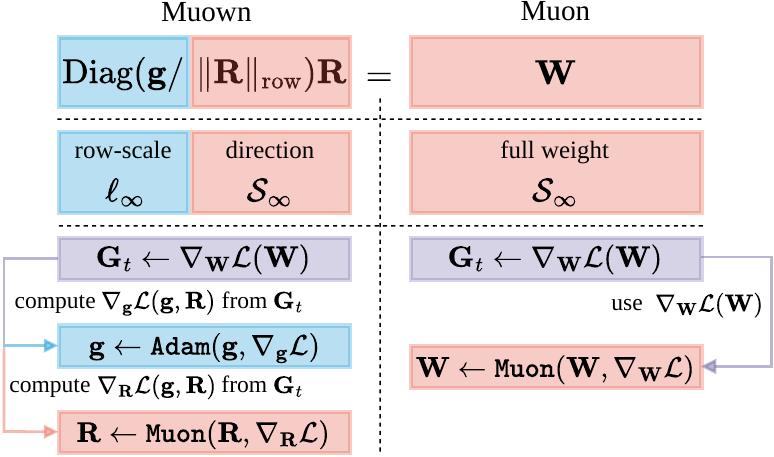}
         \caption{Abstract Illustration of \algname{} and Muon, showing the parameterization (top), semantics and geometry (middle), and the optimizer step (bottom).}
         \vspace{-15pt}
         \label{fig:muown-diagram}
\end{wrapfigure}

\textbf{Operator Norm Perspective.} This viewpoint is formalized by the operator-norm framework of metrized deep learning \citep{large_scalable_2024, bernstein_modular_2024}, where each layer is equipped with a geometry derived from its input-output semantics, and gradients, treated as dual vectors, are transported back to the primal weight space by a layerwise duality map \citep{nemirovski_problem_1983, beck_mirror_2003}. For MLP layers under a rescaled spectral norm, this map recovers Muon's orthogonalization \citep{bernstein_old_2024, jordan2024muon}. Both sides of the spectral condition then live in this layerwise operator norm: the size prescriptions on $\W$ and $\Delta\W$, as well as the steepest-descent rule that produces compatible updates. Steepest descent, however, only bounds $\opnorm{\Delta\W}$ per step, leaving the control of $\opnorm{\W}$ along the trajectory as a separate problem.

\begin{figure}
    \centering
    \begin{subfigure}[t]{0.5\textwidth}
        \centering
        \includegraphics[width=\textwidth]{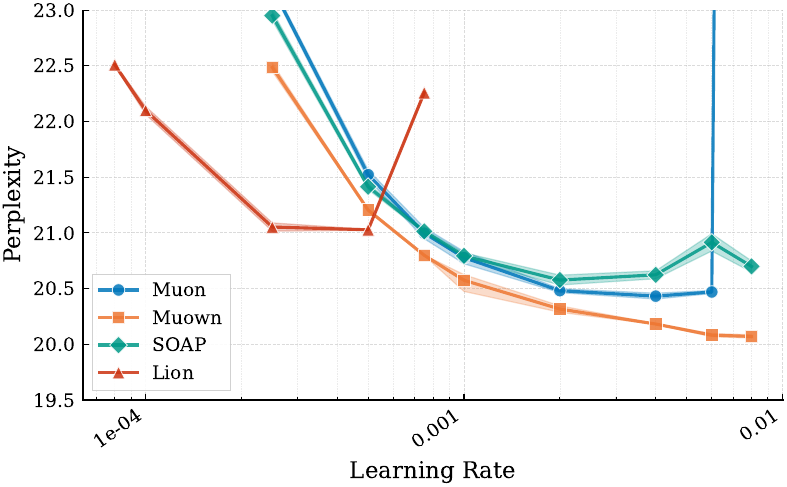}
        \caption{With weight decay}
        \label{fig:lr-ablation-160M-with-weight-decay}
    \end{subfigure}\hfill
    \begin{subfigure}[t]{0.5\textwidth}
        \centering
        \includegraphics[width=\textwidth]{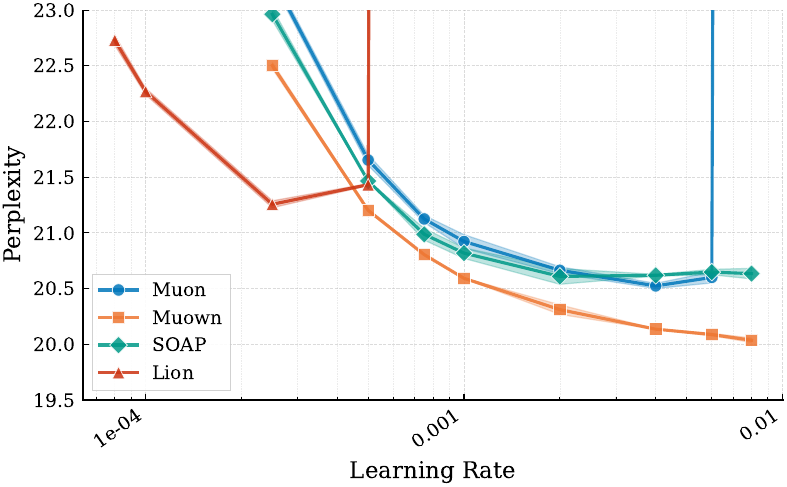}
        \caption{Without weight decay}
        \label{fig:lr-ablation-160M-without-weight-decay}
    \end{subfigure}

    \caption{Perplexity as a function of learning rate for Muon, \algname{}, Lion, and SOAP on a 124M model. \algname{} achieves lower perplexity across all learning rates and is more robust to the choice of learning rate, maintaining strong performance even at higher values where Muon diverges. We repeat the average of two seeds, reporting min and max as shaded area. Details are reported in Appendix~\ref{appdx:experimental-details}. }
    \label{fig:lr-ablation}
    \vspace{-15pt}
\end{figure}

\textbf{Spectral Control.} Existing mechanisms for keeping $\opnorm{\W}$ bounded during training fall into three broad families. \emph{Decoupled weight decay} shrinks $\W$ uniformly at every step. \cite{chen_muon_2025} cast Muon with decoupled weight decay as a member of the Lion-$\mathcal{K}$ family \citep{chen_lion_2025} whose trajectories converge to KKT points of a spectral-norm-constrained problem, with related Frank–Wolfe interpretations developed in \citep{sfyraki_lions_2025, pethick_training_2025}. \emph{Spectral reparameterizations} hardwire the singular spectrum of $\W$, e.g.\ as orthogonal under Parseval/Stiefel parameterizations \citep{cisse_parseval_2017}. \emph{Spectral-norm caps} act on the top singular value via soft penalties \citep{yoshida_spectral_2017}, post-hoc normalization \citep{miyato_spectral_2018}, or odd-polynomial clipping \citep{newhouse_training_2025}. All three families act on the singular spectrum of $\W$ as an aggregate quantity, leaving open which component of $\W$ \emph{empirically} drives spectral-norm growth under Muon.

Our work answers this open question at the parameterization level. We identify the maximal row magnitude $\|\g_t\|_\infty$ as the empirical driver of $\opnorm{\W_t}$ drift under Muon (Section~\ref{sec:motivation}) and reparameterize $\W$ so that this driver is promoted to an explicit optimizer variable, updated under its natural $\ell_\infty$ geometry. Spectral-norm control thus emerges from the construction of the parameterization rather than from an aggregate constraint on $\W$, avoiding both the indiscriminate shrinkage of decoupled weight decay and the orthogonality prior of spectral reparameterizations.

\section{Method}
\label{sec:method}

\paragraph{Notation.} Let $\W_t \in \mathbb{R}^{m \times n}$ denote the current iterate of a weight matrix of a linear layer with $m$ and $n$ output and input neurons, respectively, and $\G_t = \nabla_\W \mathcal{L}(\W_t) \in \mathbb{R}^{m \times n}$ its corresponding gradient at iteration $t$ with respect to a scalar loss $\mathcal{L}(\cdot)$. Division by vectors is elementwise. We write $\mathrm{diag}(\mathbf{A})\in\mathbb{R}^{k}$ for the vector of diagonal entries of a square matrix $\mathbf{A}\in\mathbb{R}^{k\times k}$, and $\mathrm{Diag}(\mathbf{a})\in\mathbb{R}^{k\times k}$ for the diagonal matrix with diagonal given by $\mathbf{a}\in\mathbb{R}^{k}$. For $p \in [1,\infty]$, $\norm{\bfa}_p$ denotes the $\ell_p$-norm of a vector $\bfa \in \mathbb{R}^k$, and $\norm{\A}_{S_p}$ the Schatten-$p$ norm of a matrix $\A\in\mathbb{R}^{m\times n}$, i.e.\ the $\ell_p$-norm of its singular values $\sigma_1(\A)\geq\dots\geq\sigma_{\min(m,n)}(\A)\geq 0$. In particular $\opnorm{\A}$ is the spectral norm and $\trnorm{\A}$ the nuclear norm. The canonical (Frobenius) inner product on $\mathbb{R}^{m\times n}$ is $\langle \A,\B\rangle := \mathrm{tr}(\A^\top \B)$ with induced norm $\normF{\A}$, and $\langle a, b \rangle = a^\top b$ on $\mathbb R^k$. We write $\lambda_{\max}(\mathbf M)$ for the largest eigenvalue of a symmetric matrix $\mathbf M$. With slight abuse of notation, we denote the row-norm vector $\RowNorm{\A} = (\norm{\A_{1,:}}_2, \dots, \norm{\A_{m, :}}_2)^\top \in \mathbb R^m$. For $\A \in \mathbb{R}^{m \times n}$, we define the projection which removes the per-row radial component with respect to $\X \in \mathbb{R}^{m \times n}$ as $\mathrm{Proj}_{\mathbf{X}}(\mathbf{A}) := \mathbf{A} - \mathrm{Diag}\!\big(\mathrm{diag}(\mathbf{A}\mathbf{X}^\top)\big)\,\mathbf{X}.$ In other words, for each row $\abf_i$, $\mathrm{Proj}_{\mathbf{X}}(\mathbf{A})$ subtracts $\langle \mathbf{a}_i,\mathbf{x}_i\rangle\mathbf{x}_i$.

\subsection{Motivation}
\label{sec:motivation}
\textbf{Row Norm Drift.} The constrained optimization and metrized-deep-learning arguments in Section~\ref{sec:related-work} single out $\opnorm{\W_t}$ as the quantity that governs stability of feature propagation,
yet steepest spectral descent only bounds $\opnorm{\Delta\W_t}$ per step. Before prescribing a separate control mechanism for $\opnorm{\W_t}$ itself, we ask a more basic question: \emph{which component of $\W_t$ is responsible for the observed spectral-norm drift under Muon?} The following decomposition isolates two candidate mechanisms.

\begin{proposition}[Spectral Norm Decomposition into Row Magnitude and Coherence]
\label{prop:spectral-norm-decomposition}
Let $\W_t \in \mathbb{R}^{m\times n}$ and suppose each row of $\W_t$ is nonzero. Define the row-magnitude vector $
\g_t := \RowNorm{\W_t}\in \mathbb{R}^m,
$
and the row-normalized matrix $
\D_t := \mathrm{Diag}\!\left(1/\g_t\right)\W_t \in \mathbb{R}^{m\times n},
$
so that $
\W_t=\mathrm{Diag}(\g_t)\D_t.
$
Let $
\C_t := \D_t\D_t^\top \in \mathbb{R}^{m\times m}
$
be the Gram matrix of the unit-norm rows, and define $
\p_t := |\g_t|/\norm{\g_t}_\infty \in [0,1]^m$ with $\PP_t := \mathrm{Diag}(\p_t).$
Then the spectral norm of $\W_t$ admits the decomposition 
\begin{equation*}
\opnorm{\W_t}^2
=
\norm{\g_t}_\infty^2\,\lambda_{\max}\!\bigl(\PP_t\,\C_t\,\PP_t\bigr).
\end{equation*}
In particular, the spectral norm splits into a row-scale term $\norm{\g_t}_\infty^2$ and a normalized alignment term $\lambda_{\max}(\PP_t\C_t\PP_t)$.
\end{proposition}

The proof is deferred to Appendix~\ref{ref:proof-spectral-norm-decomposition}. We remark that the two factors are \textit{disjoint} by design: any positive uniform row-rescaling of
$\W_t$ moves only $\|\g_t\|_\infty$, while any directional change to the rows
moves only $\lambda_{\max}(\PP_t\C_t\PP_t)$. The alignment factor $\lambda_{\max}(\PP_t\C_t\PP_t)\ge 1$ measures coherence of the unit-norm neurons in $\D_t$, weighted by the normalized row-scale profile $\p_t$. It collapses to $1$ if and only if the (non-negligible) rows are orthonormal, in which case $\opnorm{\W_t}=\|\g_t\|_\infty$ and the layer's operator norm is determined entirely by its maximal row scale.

\textbf{Row Norm Control.} Across the attention and MLP linear layers of a 500M model trained with Muon, we find that $\opnorm{\W_t}$ and $\|\g_t\|_\infty$ move in lock-step throughout training (Fig.~\ref{fig:row-scales-mlp},~\ref{fig:row-scales-wout}; Fig.~\ref{fig:spectral-norm-detailed}, \ref{fig:row-norm-detailed}), whereas $\lambda_{\max}(\PP_t\C_t\PP_t)$ only exhibits a brief transient rise early in training, remaining bounded thereafter (Fig.~\ref{fig:lambda-max-fc2},~\ref{fig:lambda-max-w_out}; Fig.~\ref{fig:lambda-max-detailed}). A simple causal intervention confirms that the row scales drive the spectral norm drift: freezing $\g_t =\g_1$ after initialization, a minimal intervention that touches only the row-scale factor, eliminates the systematic growth of $\opnorm{\W_t}$ (Figs.~\ref{fig:sv_histograms}, \ref{fig:spectral-norm-detailed}).

With $\|\g_t\|_\infty$ identified as the empirical driver of $\opnorm{\W_t}$ under Muon, we respond by promoting it from a byproduct of the $\W$-update to a trainable variable, optimized under its own geometry. Existing spectral-control schemes instead intervene at coarser granularities of Proposition~\ref{prop:spectral-norm-decomposition}. Trainable row magnitudes, by contrast, act only on the identified driver $\|\g_t\|_\infty$, without the orthogonality prior of reparameterizations and without the indiscriminate shrinkage of weight decay.

\begin{figure}[t]
        \centering
    \begin{subfigure}[t]{0.49\textwidth}
        \includegraphics[width=\textwidth]{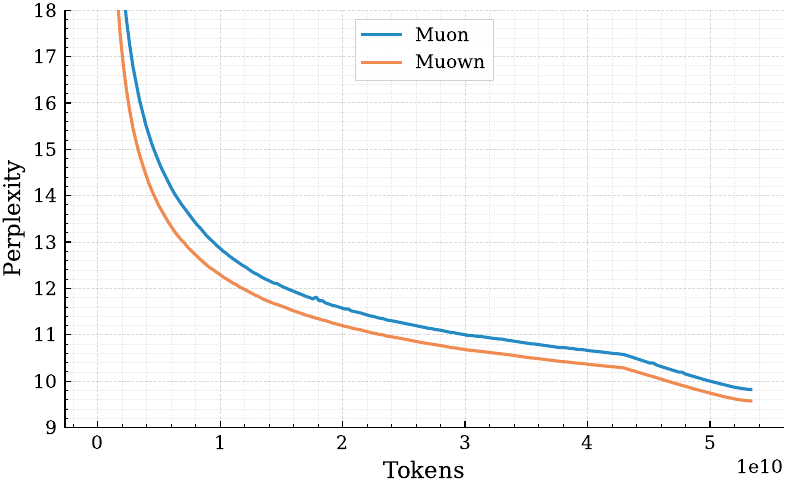}
        \caption{2.7B Model}
        \label{fig:3B-plot}
    \end{subfigure}\hfill
    \begin{subfigure}[t]{0.49\textwidth}
        \includegraphics[width=\textwidth]{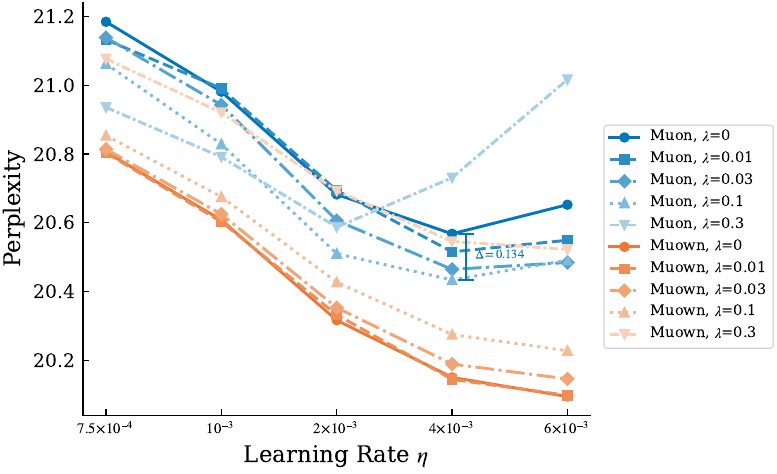}
        \caption{Weight Decay Ablation for 124M}
        \label{fig:weight-decay-ablation}
    \end{subfigure}
    \caption{\emph{Left:} Perplexity of Muon and \algname{} on a 2.7B transformer trained on FineWeb-Edu. \emph{Right:} Ablation of weight decay and learning rate for Muon and \algname{} on 124M models after 5B tokens.}
    \vspace{-15pt}
\end{figure}

\subsection{\algname{}}
\label{sec:algname}

To turn $\|\g_t\|_\infty$ into a controllable variable, we reparameterize each linear layer as $\W(\g,\Rbf) = \Diag(\g/\RowNorm{\Rbf})\Rbf$ with $\g\in\mathbb{R}^m$ and $\Rbf\in\mathbb{R}^{m\times n}$, i.e.\ weight normalization \citep{salimans_weight_2016}, under which $\g$ coincides exactly with the row-magnitude vector of Proposition~\ref{prop:spectral-norm-decomposition}. To remain agnostic to the model implementation, we realize the reparameterization \emph{implicitly inside the optimizer}: $\g$ is held as an additional optimizer state, $\Rbf$ is reconstructed from the stored effective weight $\W$ on each step, and the forward pass is left unchanged. The resulting procedure (Algorithm~\ref{alg:algname}, Fig.~\ref{fig:muown-diagram}) is therefore a drop-in replacement for Muon, which decomposes each layer into a direction $\Rbf$, already metrized by the spectral norm, and a diagonal per-neuron gain $\mathrm{Diag}(\g)$.

It remains to choose a geometry for $\g$. Composability in the modular-duality framework \citep[Def.~6]{bernstein_modular_2024} requires that the output space of $\Rbf$ and the input space of $\mathrm{Diag}(\g)$ carry the same norm, so the natural metric on $\g$ is the induced operator norm $\opnorm{\mathrm{Diag}(\Delta\g)}$. Since the operator norm of a diagonal matrix equals the largest entry in absolute value, this yields $\opnorm{\mathrm{Diag}(\Delta\g)}=\|\Delta\g\|_\infty$, i.e., the $\ell_\infty$ geometry on $\g$. This is independently confirmed by Proposition~\ref{prop:spectral-norm-decomposition}, which identifies $\|\g_t\|_\infty$ as the row-scale factor of $\opnorm{\W_t}$.

Steepest descent under $\ell_\infty$ is \signsgd\ \citep{bernstein_signsgd_2018}, with the entrywise sign as the duality map \citep[Ex.~2]{bernstein_modular_2024}. Its update $\Delta\g = -\eta\,\operatorname{sgn}(\nabla_\g)$ satisfies $\|\Delta\g\|_\infty = \eta$ by construction, so the row-scale contribution $\opnorm{\mathrm{Diag}(\Delta\g)} = \|\Delta\g\|_\infty$ to $\opnorm{\Delta\W_t}$ is bounded at every step. In practice, \signsgd\ is the stateless analogue of Adam with both EMAs disabled \citep{balles_geometry_2020, xie_implicit_2024}, and AdamW has been shown to converge to KKT points of an $\ell_\infty$-constrained problem \citep{xie_implicit_2024}. We therefore adopt Adam as the magnitude optimizer in Algorithm~\ref{alg:algname}, as it inherits the $\ell_\infty$ geometry while adding momentum-based smoothing against stochastic-gradient noise.

We initialize $\Rbf_1 \gets \W_1$ and cached row norms and magnitudes as $\rbf, \g_1 \gets \RowNorm{\W_1}$, so that $\W(\g_1,\Rbf_1)=\W_1$, making \algname{} start from exactly the same point as Muon under any non-zero initialization scheme. All remaining optimizer states are zero. Following \citet{liu_muon_2025}, we scale the $\Rbf$-step by $0.2\sqrt{\max(m,n)}$ to match the RMS norm of a typical Adam update, so that a single learning rate $\eta_t$ drives both the direction and magnitude updates and no separate learning rate for $\g$ is introduced. The resulting algorithm can be found in Algorithm~\ref{alg:algname}.

\section{Convergence Analysis}
\label{sec:conv_analysis}

We establish that \algname{} attains the optimal non-convex rate in the stochastic setting (Theorem~\ref{thm:stoch_convergence}) in a dual norm matching the two update geometries. 

Following the prevalent view of Adam as a smoothed sign method \citep{kunstner2023noise}, we analyze SignSGD with momentum (Signum) as its proxy on $\g$. SignSGD preserves the relevant $\ell_\infty$ geometry dynamics which largely explain Adam's behavior in Transformer training \citep{kunstner2023noise}, while making the underlying $\ell_\infty$ geometry explicit \citep{bernstein_signsgd_2018,balles_geometry_2020,xie_implicit_2024}.

The deterministic case follows and all proofs are deferred to Appendix~\ref{appdx:proofs}.

\begin{wrapfigure}{r}{0.51\textwidth}
  \vspace{-20pt}
  \begin{minipage}{0.51\textwidth}
    \begin{algorithm}[H]
    \caption{Step of Muon with integrated Weight Normalization (\algname{})}\label{alg:algname}
    {\linespread{1.3}\selectfont
    \begin{algorithmic}
    \Require $\W \in \mathbb{R}^{m \times n}$, gradient $\nabla_{\W} \mathcal{L}(\W) \in \mathbb{R}^{m \times n}$,
    states $\g, \rbf, \m_{\g}, \vbf_{\g} \in \mathbb{R}^{m}$, $\M \in \mathbb{R}^{m \times n}$, learning rate $\eta_t$, weight decay $\lambda$, momentum $\beta_1$
    \State $\Rbf \gets \mathrm{Diag}(\frac{\rbf}{\g}) \W$
    \State $\D \gets \mathrm{Diag}(1/\rbf)\Rbf$
    \State \textcolor{cyan}{\texttt{\# obtain decoupled gradients}}
    \State $\nabla_{\g} \mathcal{L}(\W) \gets (\nabla_{\W} \mathcal{L}(\W) \odot\D) \mathbf{1}_n$
    \State $\nabla_{\Rbf} \mathcal{L}(\W) \gets \mathrm{Diag}(\frac{\g}{\rbf})\mathrm{Proj}_{\D}(\nabla_{\W}\mathcal{L}(\W))$
    \State \textcolor{cyan}{\texttt{\# run muon w/ nesterov on direction}}
    \State $\M \gets \beta_1 \M + \nabla_{\Rbf} \mathcal{L}(\W)$
    \State $\Obf \gets \arg\min_{\opnorm{\Obf} \le 1}\langle \beta_1\M + \nabla_{\Rbf} \mathcal{L}(\W),\Obf\rangle$
    \State $\Rbf \gets \Rbf + 0.2 \sqrt{\max(m, n)}\eta_t \Obf$
    \State \textcolor{cyan}{\texttt{\# run adam on magnitude}}
    \State $\g  \gets \mathrm{Adam}(\nabla_{\g} \mathcal{L}(\W), \m_{\g}, \vbf_{\g}, \eta_t)$
    \State \textcolor{cyan}{\texttt{\# obtain effective weight}}
    \State $\rbf \gets \|\Rbf\|_{\mathrm{row}}$
    \State $\W \gets \mathrm{Diag}(\frac{\g}{\rbf}) \Rbf$
    \end{algorithmic}
    }
\end{algorithm}
  \end{minipage}
  \vspace{-30pt} 
\end{wrapfigure}

\textbf{Setup and Geometry.}\ Working in the re-parameterized variables $\g \in \mathbb{R}^m$, $\Rbf \in \rowPos$, where $\rowPos$ denotes the set of $m\times n$ matrices with strictly positive row norms, we consider $\Ft \colon \mathbb{R}^m \times \rowPos \to \mathbb{R}$, $\Ft(\g,\Rbf) := \mathcal L\pare{\W\pare{\g, \Rbf}} = \mathcal L\!\left(\Diag\!\left(\frac{\g}{\RowNorm{\Rbf}}\right)\, \Rbf\right)$. Since the $\g$- and $\Rbf$-updates take unit-sized steps in the $\ell_\infty$- and spectral norms respectively, we equip the joint variable with the product max-norm $\norm{(\g, \Rbf)} = \max\{\norm{\g}_\infty, \opnorm{\Rbf}\}$ and the canonical inner product $\langle (\g, \Rbf), (\g', \Rbf')\rangle := \langle \g, \g'\rangle + \langle \Rbf, \Rbf'\rangle$. The corresponding dual norm is $\norm{(\g, \Rbf)}_* = \norm{\g}_1 + \trnorm{\Rbf}$, combining the $\ell_1$-norm on $\mathbb{R}^m$ with the Schatten-$1$ (nuclear) norm on $\mathbb{R}^{m\times n}$. Each summand is the dual of the norm in which the corresponding update is bounded, so the product norm simply lifts the modular-duality viewpoint of Section~\ref{sec:related-work} to the joint variable $(\g,\Rbf)$.

\textbf{Stochastic Convergence.} We consider the objective $\mathcal L(\W) := \mathbb{E}_{\xi}[\mathcal L(\W;\xi)]$ with access only to stochastic gradients $\nabla \mathcal L(\W, \xi)$, and use the re-parametrized gradient oracles $\nabla_{\g} \Ft(\gt, \Rt, \xi_t) = \diag(\nabla \mathcal L(\W(\gt, \Rt), \xi_t) \D_t^\top)$ and $\nabla_{\Rbf} \Ft(\gt, \Rt, \xi_t) = \Diag\pare{\frac{\gt}{\RowNorm{\Rt}}}\mathrm{Proj}_{\D_t}(\nabla \mathcal L(\W(\gt, \Rt), \xi_t))$, where $\D_t = \Diag(1/\RowNorm{\Rbf_t}) \Rbf_t $ is the row-normalization of $\Rbf_t$. 
\begin{figure}[t]
    \centering
    \includegraphics[width=0.85\linewidth]{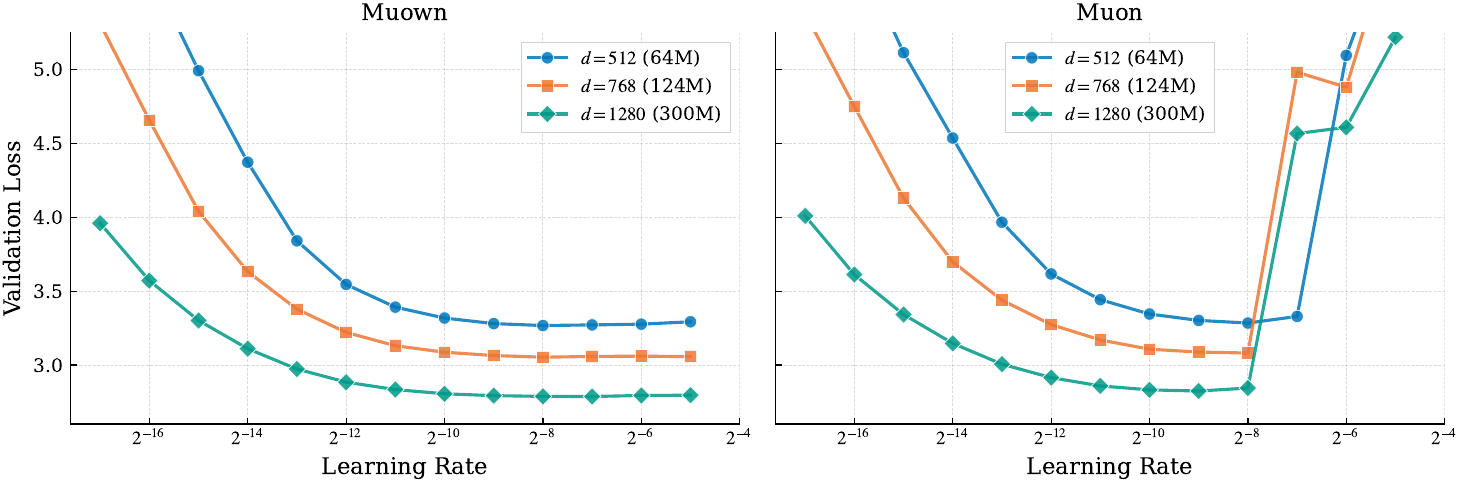}
    \caption{Validation loss of Muon and \algname{} at three different model widths of 512, 768, 1280 with 12 layers each. We observe that \algname{} yields close to optimal validation loss for a wide range of learning rates from $2^{-10}$ up to $2^{-5}$, irrespective of model width. }
    \label{fig:width-ablation}
    \vspace{-25pt}
\end{figure}

\begin{theorem}[Stochastic Convergence of \algname{}]\label{thm:stoch_convergence}
    Assume $\Ft$ is $L$-smooth along the trajectory, lower bounded by $\Ft_\star$ and the gradient oracles have $\sigg$ and $\sigR$ bounded variance accordingly (see Assumption \ref{assum:bounded_var}). Denote $\Delta_1 \geq \Ft(\g_1, \Rbf_1) - \Ft_\star$ and the norm equivalence constants $\zeta_\g := \max_\g \nicefrac{\norm{\g}_1}{\norm{\g}_2}, \zeta_\Rbf := \max_{\Rbf} \nicefrac{\trnorm{\Rbf}}{\normF{\Rbf}}$. Then the iterates of Algorithm \ref{alg:muown_signgd} with 
    $
        \eta_t \equiv \sqrt{\frac{\Delta_1 (1-\beta_1)}{L T}} \text{ and }
        \beta_1 = 1 - \min\left\{1, \max\left\{T^{-\nicefrac 2 3}, \sqrt{\frac{\Delta_1 L}{\hat\sigma^2 T}}\right\}\right\},
    $
    where $\hat \sigma := \zeta_\g \sigg + \zeta_\Rbf \sigR$, satisfy
    \begin{equation*}
    \frac{1}{T}\sum_{t=1}^T \E\left[\norm{\nabla \Ft\pare{\gt,\Rt}}_*\right]
    \leq 
    6 \sqrt{\frac{\Delta_1 L}{T}}
    + 8 \pare{\frac{\Delta_1 L \hat \sigma^2}{T}}^{\nicefrac 1 4}
    + \frac{4 \hat \sigma}{T^{\nicefrac 1 3}}
\end{equation*}
    which matches the optimal $\mathcal O\pare{\pare{\frac{\Delta_1 L \sigma^2}{T}}^{\nicefrac 1 4}}$ non-convex rate up to constants.
\end{theorem}

The guarantee is stated for the dual norm of the re-parameterized gradient, and the dominant term matches the optimal stochastic non-convex rate up to constants \citep{LowerBoundsNonConvex2023arjevani}, mirroring the rate established for Muon by \cite[Corollary~4.4]{ConvergenceAnalysisMuon2025shen}. Moreover, since $\norm{\cdot}_1 \geq \norm{\cdot}_2$ and $\trnorm{\cdot} \geq \normF{\cdot}$, the bound is strictly stronger than its $\ell_2$/Frobenius counterpart. We caution that the bound concerns $\nabla \Ft$. Although $\Ft(\g, \Rbf) = \mathcal L(\W)$ pointwise, $\nabla \Ft$ and $\nabla \mathcal L$ need not be directly comparable. The only algorithm-dependent factor is the noise coefficient $\hat\sigma = \zeta_\g \sigg + \zeta_\Rbf \sigR$, where the norm-equivalence constants $\zeta_\g, \zeta_\Rbf$ arise, as in recent Muon analyses \citep[Theorem~4.3]{ConvergenceAnalysisMuon2025shen}, from controlling the stochastic error in $\ell_1/\text{Schatten-}1$ from $\ell_2/\text{Frobenius}$ variance bounds. In comparison, the corresponding factor for Muon is $\zeta_\W \sigma_\W$, hence our theory predicts an improvement of \algname{} over Muon whenever $\zeta_\g \sigg + \zeta_\Rbf \sigR < \zeta_\W \sigma_\W$. Section~\ref{sec:experiments} and Fig.~\ref{fig:scaled_cm_comparison} verify this inequality empirically throughout training, with the gap widening up to a factor of $1.6\times$, offering reduced gradient noise in the $(\g,\Rbf)$-parameterization as a candidate mechanism for explaining the observed perplexity gains.

In the proof, we first use our choice of geometry to split the descent Lemma into independent $\g$- and $\Rbf$-terms (Lemma~\ref{lem:split_algs}). We then apply the steepest-descent analyses of \signsgd{} and Muon, and recombine without cross-terms. The full proof can be found in Appendix~\ref{sec:app.conv_analysis}.

\section{Experiments}
\label{sec:experiments}

In this section, we evaluate \algname{} on language-model pre-training with modern GPT-style architectures on FineWeb-Edu. This section is designed to answer \emph{(i)} whether \algname{} can improve convergence relative to current state-of-the-art geometry-aware optimizers at a fixed token budget, \emph{(ii)} whether it reduces sensitivity to learning rate and weight decay tuning, and \emph{(iii)} whether its convergence behavior is better conditioned, as indicated by gradient rank and gradient noise. We conclude this section with a discussion of the computational overhead. 

\begin{wraptable}{r}{0.58\textwidth}
        \centering
        \vspace{-10pt}
        \caption{Perplexity for \algname{}, Muon, and SOAP on a 500M model for different learning rates. We report average and standard deviation across 3 seeds.}
        \label{tab:500M-multiseed}
        \resizebox{\linewidth}{!}{
        \begin{tabular}{l cc cc cc}
            \toprule
            $\eta$ &
            \multicolumn{2}{c}{\algname{}} &
            \multicolumn{2}{c}{Muon} & \multicolumn{2}{c}{SOAP} \\
            \cmidrule(lr){2-3}\cmidrule(lr){4-5} \cmidrule(lr){6-7}
            & \multicolumn{2}{c}{$\lambda=0.0$} 
            & $\lambda=0.1$ & $\lambda=0.0$ & \multicolumn{2}{c}{$\lambda=10^{-4}$} \\
            \midrule
            $4\times 10^{-3}$   & \multicolumn{2}{c}{$14.13_{\pm 0.004}$} & $28.37_{\pm 23.59}$ & $88.15_{\pm 3.869}$ & \multicolumn{2}{c}{$14.53_{\pm 0.017}$} \\
            $2\times 10^{-3}$   & \multicolumn{2}{c}{$14.14_{\pm 0.027}$} & $14.37_{\pm 0.028}$ & $14.40_{\pm 0.042}$ & \multicolumn{2}{c}{$14.47_{\pm 0.022}$} \\
            $1\times 10^{-3}$   & \multicolumn{2}{c}{$14.24_{\pm 0.026}$} & $14.44_{\pm 0.039}$ & $14.54_{\pm 0.041}$ & \multicolumn{2}{c}{$14.49_{\pm 0.036}$}\\
            \bottomrule
        \end{tabular}%
        }
    \vspace{-10pt}
\end{wraptable}

\textbf{Setup.} We adopt the experimental setup of \cite{ajroldi2024plainlm} which is based on a nanoGPT \citep{karpathy2022nanogpt} implementation augmented with recent architectural improvements such as RoPE \citep{su_rope}, RMSNorm normalization \citep{zhang_rmsnorm}, and SwiGLU activations \citep{shazeer_glu_2020}. We use a warmup-stable-decay learning rate schedule \citep{hu_minicpm_2024} and train most models at four sizes: 124M, 500M, 1B, and 2.7B. We focus on perplexity at a fixed token budget as a primary evaluation criterion. All training details can be found in Appendix~\ref{appdx:experimental-details}.

\textbf{Pre-training Perplexity.} We organize pre-training results by scale: a full learning-rate sweep at 124M (Fig.~\ref{fig:lr-ablation}), a multi-seed comparison at 500M (Table~\ref{tab:500M-multiseed}), and single-seed runs at 1B and 2.7B (Table~\ref{tab:1B-results}, Fig.~\ref{fig:3B-plot}). Unless explicitly stated otherwise,
\algname{} is run without weight decay ($\lambda=0$). For Muon we use $\lambda=0.1$. Both of these we find to be optimal across $\lambda \in \{0, 0.01, 0.03, 0.1, 0.3\}$ on the 124M grid (Fig.~\ref{fig:weight-decay-ablation}; discussed in detail below). We additionally include $\lambda = 0$ for Muon. 

At the 124M scale, Fig.~\ref{fig:lr-ablation} compares \algname{} against Muon, SOAP, and Lion across eight log-spaced learning rates, both with and without weight decay on the Muon-family optimizers. \algname{} attains the lowest perplexity at \emph{every} learning rate in the grid and degrades gracefully at the top of the range, whereas Muon diverges for learning rates at which \algname{} is still near-optimal.

At 500M (Table~\ref{tab:500M-multiseed}), \algname{} retains at least a $0.2$ perplexity improvement over the best Muon configuration at every matched learning rate, significantly larger than the per-seed standard deviation of stable non-divergent runs. The stability gap widens at $\eta=4\times 10^{-3}$: Muon diverges ($88.15$ at $\lambda=0$) or exhibits large seed variance ($28.37 \pm 23.59$ at $\lambda=0.1$), while \algname{} gets closer to optimality, corroborating the robustness signature seen at 124M.

The gap persists at the 1B and 2.7B scale. At 1B (Table~\ref{tab:1B-results}), \algname{} attains $11.90$ perplexity at $\eta=3\times 10^{-3}$, improving upon the best Muon run ($12.20$ at $\eta=2\times 10^{-3}$, $\lambda=0$) by $0.30$. At 2.7B (Fig.~\ref{fig:3B-plot}), \algname{} reaches $9.57$ at $\eta=2\times 10^{-3}$ while the best stable Muon run lands at $9.82$ ($\eta=10^{-3}$, $\lambda=0$), preserving a $0.25$ perplexity advantage for \algname{}. The consistent perplexity gap from 500M up to 2.7B suggests that the improvement does not attenuate with scale within the regime we reach.

Two additional experiments probe whether these gains transfer beyond the architecture and the four baselines above. On a Qwen2-0.5B \citep{yang_qwen2_2024} architecture (Appendix~\ref{sec:qwen2-exp}), \algname{} preserves a $0.25$ perplexity advantage over the best-tuned Muon configuration. Moreover, compared to NorMuon \citep{li_normuon_2025}, a per-row Muon variant that adapts the orthogonalized momentum via a row-wise second-moment statistic (Appendix~\ref{appdx:normuon-comparison}), \algname{} matches or improves perplexity in the near-optimal range and exhibits a wider basin of stability at 124M.

\begin{wraptable}{r}{0.62\textwidth}
    \centering
    \small
    \vspace{-10pt}
    \setlength{\tabcolsep}{6pt}
    \renewcommand{\arraystretch}{1.05}
    \caption{Perplexity for \algname{} and Muon on 1B and 2.7B models across learning rates. We also report perplexity of AdamW at 1B. $^{\dagger}$Run aborted after 12h on 16 GPUs as validation loss was diverging.}
    \label{tab:1B-results}
    \begin{tabular}{l l c cc c}
        \toprule
        Scale & $\eta$
         & \algname{} & \multicolumn{2}{c}{Muon} & AdamW \\
        \cmidrule(lr){3-3}\cmidrule(lr){4-5}\cmidrule(lr){6-6}
         &
         & $\lambda=0$ & $\lambda=0.1$ & $\lambda=0$ & $\lambda= 0.1$ \\
        \midrule
        \multirow{4}{*}{1B}
          & $3   \times 10^{-3}$ & 11.90 & 44.89 & 110.19 & 303.16 \\
          & $2   \times 10^{-3}$ & 11.93 & 12.33 &  12.20 &  12.69 \\
          & $1   \times 10^{-3}$ & 12.01 & 12.27 &  12.32 &  12.93 \\
          & $5   \times 10^{-4}$ & 12.18 & 12.35 &  12.44 &  12.85 \\
        \midrule
        \multirow{3}{*}{2.7B}
          & $2   \times 10^{-3}$ & 9.57  & --- $^\dagger$    &  --- $^\dagger$    & \phantom{--} \\
          & $1   \times 10^{-3}$ & 9.63  & 10.01 &  9.82  & \phantom{--} \\
          & $7.5 \times 10^{-4}$ & 9.66  & 9.96 &  9.87 & \phantom{--} \\
        \bottomrule
    \end{tabular}
\end{wraptable}
\textbf{Robustness to Hyperparameter Choices.} We assess robustness along two axes: width-dependence of the optimal learning rate (Fig.~\ref{fig:width-ablation}) and sensitivity to weight decay
(Fig.~\ref{fig:weight-decay-ablation}). Following the protocol of \cite{pethick_training_2025}, we sweep a
log-spaced learning-rate grid from $2^{-16}$ to $2^{-5}$ across three widths of $512$, $768$, and $1280$ (corresponding to $67$M, $124$M, and $308$M parameters) at Chinchilla-optimal token count. \algname{} yields near-optimal validation loss over a wide plateau of learning rates spanning $2^{-10}$ to $2^{-5}$ at every width, whereas Muon's optimal learning rate shifts downward with increasing width and its near-optimal plateau narrows. This indicates that \algname{} slightly alleviates the need for per-scale learning-rate retuning. To rule out an implementation-specific artefact, we reproduce the same qualitative phenomenon using the widely-used \texttt{modded-nanogpt} \citep{modded_nanogpt_2024} setup
(Fig.~\ref{fig:modded-nanogpt-learning-rate}).

\begin{figure}[t]
        \centering
    \begin{subfigure}[t]{0.32\textwidth}
        \includegraphics[width=\textwidth]{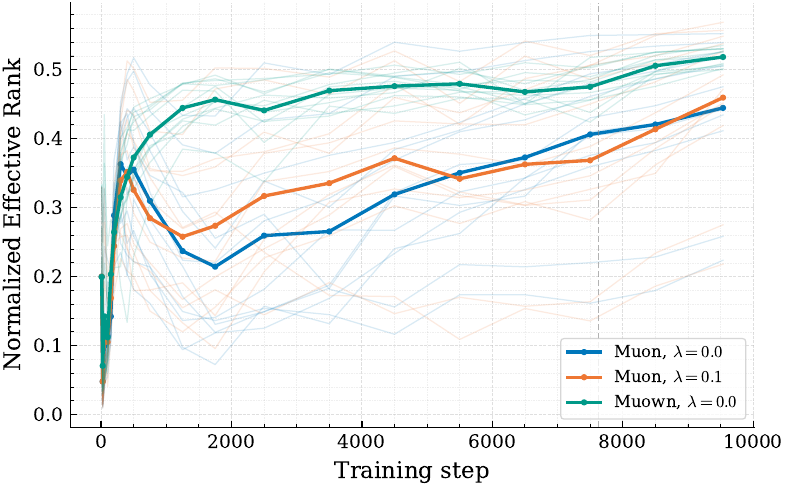}
        \caption{Query}
        \label{fig:eff-rank-w_q}
    \end{subfigure}\hfill
    \begin{subfigure}[t]{0.32\textwidth}
        \includegraphics[width=\textwidth]{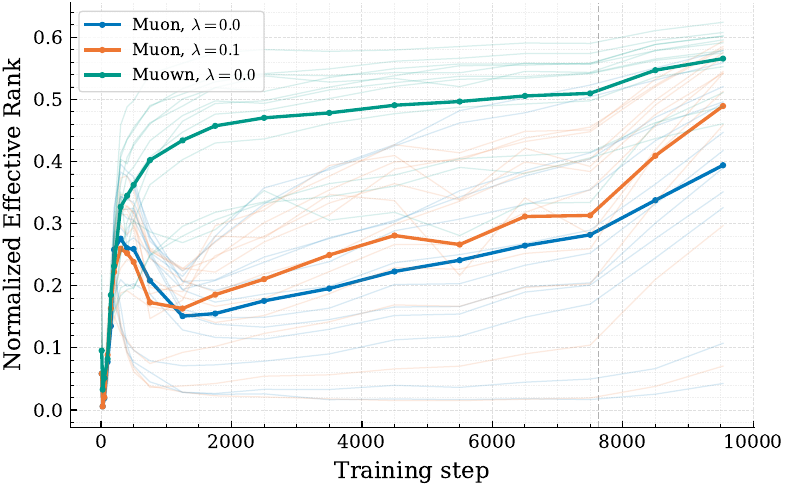}
        \caption{Output}
        \label{fig:eff-rank-w_out}
    \end{subfigure}
    \begin{subfigure}[t]{0.32\textwidth}
         \includegraphics[width=\textwidth]{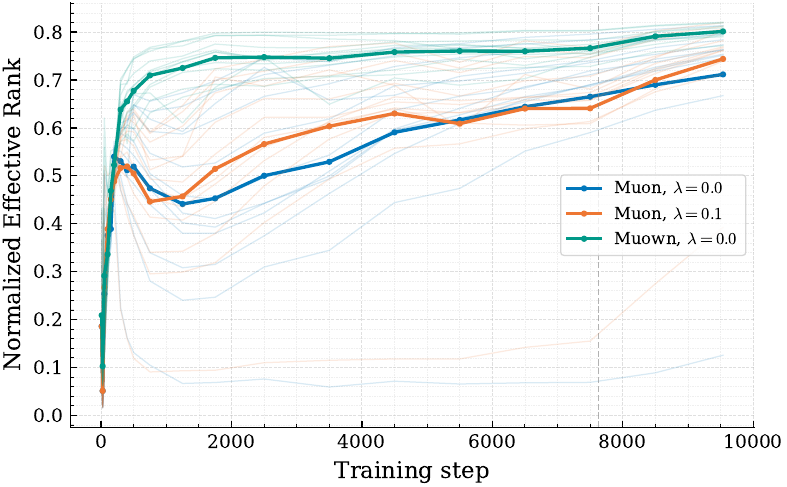}
         \caption{MLP Layer 1}
         \label{fig:eff-rank-fc1}
    \end{subfigure}\hfill
    \caption{Normalized effective rank \citep{roy_effective_2007} of minibatch gradients for a 124M model. At a given training step, we sample 8 gradients and average their effective ranks.}
    \label{fig:gradient-conditioning}
    \vspace{-20pt}
\end{figure}

Fig.~\ref{fig:weight-decay-ablation} sweeps a two-dimensional grid of learning rate and weight decay $\lambda \in \{0, 0.01, 0.03, 0.1, 0.3\}$ for Muon and \algname{} at the 124M scale. At optimally tuned learning rate, weight decay provides a sizeable improvement in perplexity for Muon and is minimized at $\lambda=0.1$, justifying the Muon baseline configuration used in the pre-training perplexity comparisons above. In contrast, the same sweep of weight decay strength yields no significant gain for \algname{}.

\textbf{Effective Rank and Stochastic Noise.} We probe two properties of the stochastic minibatch gradients each optimizer orthogonalizes, $\nabla_{\W}\mathcal L$ for Muon and $\nabla_{\Rbf}\mathcal L$ for \algname{}, to investigate the perplexity gains reported above. First, following \citet{ahn_dion_2025}, we track the normalized effective rank $\mathrm{erank}(\boldsymbol\sigma)/\min(m,n)$ with $\mathrm{erank}(\boldsymbol\sigma) = \exp(-\sum_i p_i\log p_i)$, $p_i = \sigma_i/\|\boldsymbol\sigma\|_1$ \citep{roy_effective_2007} which can be read as the fraction of directions the optimizer effectively explores. Fig.~\ref{fig:gradient-conditioning} shows that \algname{} sustains a higher effective rank than Muon on every reported weight type throughout training. Second, Section~\ref{sec:conv_analysis} showed that \algname{}'s stochastic rate scales with $\hat\sigma = \zeta_\g\sigma_\g + \zeta_\Rbf\sigma_\Rbf$ while Muon's scales with $\zeta_\W\sigma_\W$ \citep{ConvergenceAnalysisMuon2025shen}, leaving the theoretical edge conditional on $\hat\sigma < \zeta_\W\sigma_\W$. Fig.~\ref{fig:scaled_cm_comparison} verifies this inequality per weight matrix along the training trajectory: \algname{}'s noise term stays below Muon's and the gap widens over training. Together, higher spectral diversity and lower stochastic noise in the $(\g, \Rbf)$-parameterization offer a mechanistic explanation for the observed perplexity gains.

\textbf{Resource Overhead.} Table~\ref{tab:timing-overhead} reports median end-to-end step time on a 500M model (forward and backward passes, gradient synchronization, and optimizer update).
\begin{wrapfigure}{r}{0.5\textwidth}
    \includegraphics[width=0.5\textwidth]{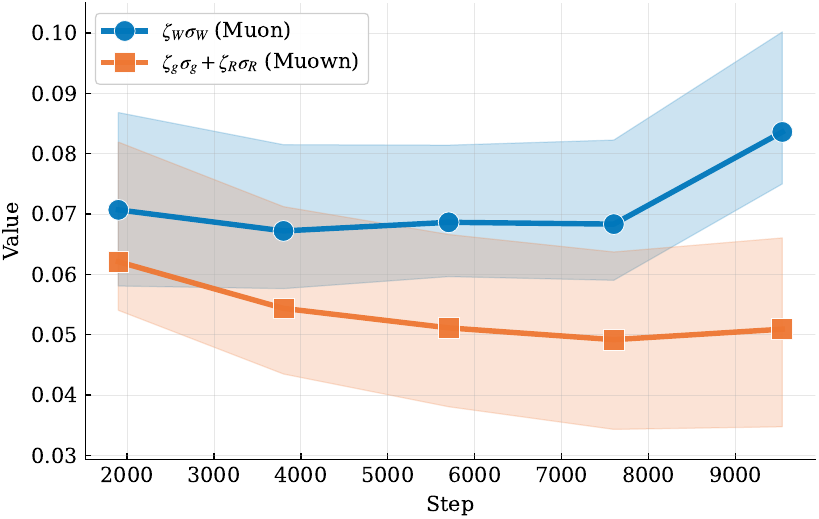}
    \caption{Comparison of the noise coefficients for Muon and \algname{}. Shaded areas mark the interquartile range over weights.}
    \vspace{-10pt}
    \label{fig:scaled_cm_comparison}
\end{wrapfigure}
Relative to Muon, \algname{} adds only elementwise $\mathcal{O}(mn)$ work per layer (the weight-norm decomposition and recomposition together with the Adam update on $\g \in \mathbb{R}^m$) asymptotically dominated by the shared Newton-Schulz iteration at $\mathcal{O}(m^2 n)$ for $m \le n$. We benchmark two variants: a \emph{replicated} baseline that runs the optimizer step on every rank, and a \emph{distributed} variant (under DDP) that follows \cite{modded_nanogpt_2024}, sharding the per-layer optimizer work round-robin across ranks and concluding with an all-gather. In the replicated setting the surplus materializes as a measurable $\approx 20\,\text{ms}$ ($1.5\%$) per step; once sharded in the same pass as Newton-Schulz, step times for Muon and \algname{} become indistinguishable, with the gap shrinking further as the rank count grows. This amortization relies on exposing the reparameterization \emph{inside} the optimizer: at the module level, every rank would redundantly recompute the decomposition on its local replica of $\W$ at every forward pass. Communication is unchanged from Muon: only $\W$ is all-gathered, while $\g$ and the reconstructed $\Rbf$ remain rank-local.

\algname{} introduces \emph{no} additional $\mathcal{O}(mn)$ state on top of Muon, since $\Rbf$ is reconstructed inside the optimizer from the stored effective weight $\W$. The only new states are four $\mathbb{R}^m$ vectors per linear layer ($\g$, the cached row norm $\|\Rbf\|_{\mathrm{row}}$, and the Adam first- and second-moment buffers for $\g$) amounting to $\approx 2\,\text{MB}$ of additional peak memory in the runs of Table~\ref{tab:timing-overhead}, negligible next to Muon's momentum buffer.

\section{Conclusion}
\label{sec:conclusion}

\begin{wraptable}{r}{0.3\textwidth}
        \vspace{-12pt}
        \centering
        \small
        \setlength{\tabcolsep}{4pt}
        \caption{Median step time (s) for \algname{} and Muon on a 500M model (1K steps, 4 GH200 GPUs). }
        \label{tab:timing-overhead}
        \begin{tabular}{l cc cc c}
            \toprule
            Setting & Muon & \algname{} \\
            \midrule
            Replicated   & 1.34 & 1.36  \\
            Distributed   & 1.27 & 1.27  \\
            \bottomrule
        \end{tabular}%
        \vspace{-10pt}
\end{wraptable} 
We show that the spectral-norm drift observed under Muon is not an irreducible property of the optimizer. Instead, the decomposition of Proposition~\ref{prop:spectral-norm-decomposition} isolates the maximal row magnitude $\|\g_t\|_\infty$ as its empirical driver, while the row-coherence factor remains stable along the trajectory. We propose \algname{}, which acts on this driver at the parameterization level, lifting $\g$ to an optimizer variable updated under the $\ell_\infty$ geometry singled out by modular duality and leaving Muon unchanged on the direction. We provide evidence that this intervention can improve perplexity, widen the plateau of near-optimal learning rates across model widths, and alleviate the need to tune weight decay. These gains come at a memory cost linear in the number of rows and negligible step-time overhead once the optimizer is sharded across ranks. 

\textbf{Limitations and Future Work.\ } The decomposition $\W=\Diag(\g/\RowNorm{\Rbf})\Rbf$ is undefined when a row of $\W$ is exactly zero at initialization, making the direction unidentifiable. Architectures that depend on zero-initialized weight matrices, such as zero-initialized residual output projections or some LoRA adapters, therefore require either a small non-zero initialization or an opt-out from the reparameterization for the affected layers. Our empirical study is also confined to dense GPT-style architectures and pre-training horizons up to 53B tokens. The behavior on convolutional, mixture-of-experts, and substantially longer-horizon regimes remains open. A first natural direction for future work is to vary the geometry imposed on $\g$. Replacing $\ell_\infty$ with sparsity-inducing  $\ell_1$-updates would transform magnitude control into structured row-wise sparsity, offering an optimizer-level route to neuron pruning and therefore potentially more efficient inference.  

\section*{Acknowledgments}
This work was supported under project ID a0184 as part of the Swiss AI Initiative, through a small grant from the ETH Domain and computational resources provided by the Swiss National Supercomputing Centre (CSCS) under the Alps infrastructure. Kai Lion is supported by Swiss National Science Foundation (SNSF) Sinergia Funding No. 216600. Florian Hübler acknowledges financial support from the ETH research grant and Swiss National Science Foundation (SNSF) Project Funding No.~200021-207343, the Alexander von Humboldt Foundation, the European Union’s Horizon Europe research and innovation programme under grant agreements No.~101120237 (ELIAS), and No.~101070617 (ELSA). Views and opinions expressed are however those of the author only and do not necessarily reflect those of the European Union or European Commission. Neither
the European Union nor the granting authority can be held responsible for them. Antonio Orvieto acknowledges the financial support of the Hector Foundation. Niao He is supported by an ETH research grant funded through the ETH Zurich Foundation and by an SNSF Starting Grant.

\bibliographystyle{plainnat}
\bibliography{refs}

\newpage
\appendix

\begin{center}
{\Large  \bf Supplementary Document for \algname{}}
\end{center}
\tableofcontents
\newpage

\section{Additional Experiments}
\label{appdx:additional-experiments}

\subsection{nanoGPT Speedrunning}

\begin{wraptable}{r}{0.47\textwidth}
        \centering
        \vspace{-10pt}
        \caption{Perplexity for \algname{} and Muon on Qwen2-0.5B pre-training.}
        \label{tab:qwen2-0.5B}
        \begin{tabular}{l cc cc}
            \toprule
            $\eta$ &
            \multicolumn{2}{c}{\algname{}} &
            \multicolumn{2}{c}{Muon} \\
            \cmidrule(lr){2-3}\cmidrule(lr){4-5} 
            & \multicolumn{2}{c}{$\lambda=0.0$} 
            & $\lambda=0.1$ & $\lambda=0.0$  \\
            \midrule
            $1\times 10^{-2}$   & \multicolumn{2}{c}{$14.57$} & -- &  -- \\
            $8\times 10^{-3}$   & \multicolumn{2}{c}{$14.59$} & $46.77$ & $94.02$ \\
            $4\times 10^{-3}$   & \multicolumn{2}{c}{$14.65$} & $14.86$ & $14.83$ \\
            $2\times 10^{-3}$   & \multicolumn{2}{c}{$14.76$} & $14.82$ & $14.88$ \\
            $1\times 10^{-3}$   & \multicolumn{2}{c}{$14.95$} & $14.91$ & $15.03$ \\
            \bottomrule
        \end{tabular}%
        
    \vspace{-20pt}
\end{wraptable}

We compare \algname{} and Muon for different learning rates and widths using the \texttt{modded-nanoGPT} codebase and report the results in Fig.~\ref{fig:modded-nanogpt-learning-rate}. Similar to \cite{pethick_training_2025}, we choose a log-space learning grid from $2^{-12}$ to $2^{-5}$. Fig.\ \ref{fig:modded-nanogpt-learning-rate} confirms the conclusions drawn in the main text regarding the greater robustness with respect to ill-chosen learning rates. 

\begin{figure}
    \centering
    \includegraphics[width=0.5\linewidth]{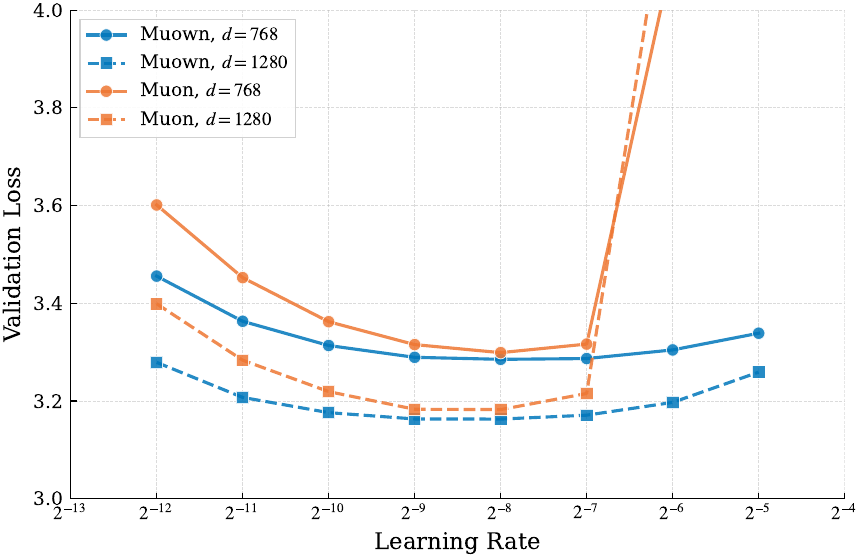}
    \caption{Learning rate and width ablation for Muon and \algname{} using the \texttt{modded-nanoGPT} codebase. }
    \label{fig:modded-nanogpt-learning-rate}
\end{figure}

\subsection{Qwen2-0.5B}
\label{sec:qwen2-exp}

As an additional architectural datapoint, we pre-train a Qwen2-0.5B \citep{yang_qwen2_2024} model, reporting results in Table~\ref{tab:qwen2-0.5B}. We remark that for all but one learning rate considered, \algname{} outperforms Muon and maintains a 0.25 perplexity advantage when comparing the best-performing hyperparameters.

\subsection{Comparison with NorMuon}
\label{appdx:normuon-comparison}

\begin{wrapfigure}{r}{0.5\textwidth}
    \includegraphics[width=0.5\textwidth]{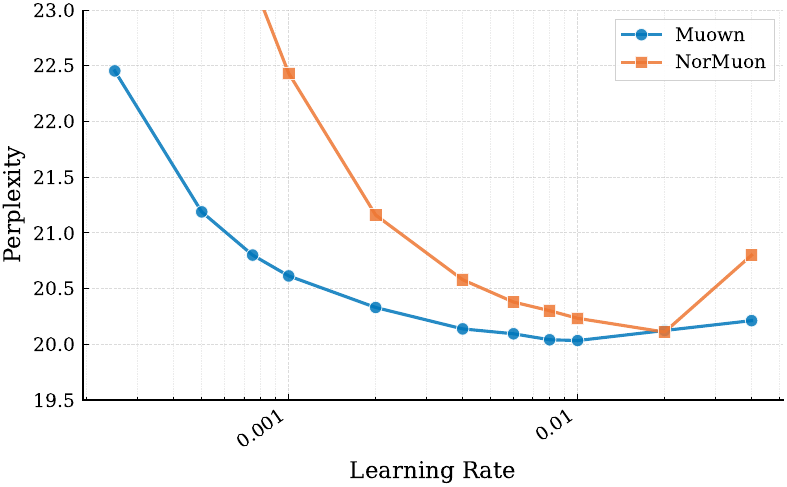}
    \caption{Comparison of NorMuon and \algname{} on a 124M model.}
    \vspace{-10pt}
    \label{fig:normuon-comparison}
\end{wrapfigure}
Both NorMuon \citep{li_normuon_2025} and Muown augment plain Muon with a per-row mechanism on the rows of the 2D weight matrices, but they intervene at different points in the update. Plain Muon orthogonalizes the momentum of $\nabla_\W \mathcal{L}$ via Newton–Schulz and applies it directly to the parameter without a row-wise state. NorMuon instead applies a post-processing step to the update, maintaining a per-row EMA of the mean-squared entries of the orthogonalized update, dividing the update element-wise by its square root, and renormalizing back to the original Frobenius norm. This procedure effectively redistributes mass of the update across rows. In contrast, rather than redistributing mass of the update \algname{} changes the parameterization to control the row-norms of the weight. In essence, both methods equip every row with a scalar quantity (a second-moment statistic for NorMuon and a learned magnitude for \algname{}). Nevertheless, these row variables serve different purposes. We compare both methods on a 124M model across eleven learning rates spanning $2.5 \times 10^{-4}$ to $4 \times 10^{-2}$. Muown matches or improves on NorMuon at most learning rates that lie near the optimum and exhibits a noticeably wider basin of stability.

\subsection{Spectral Analysis}
\label{appdx:spectral-experiments}

We provide more detailed results for the empirical association between $\norm{\W_t}_2^2$ and $\norm{\g_t}_\infty$. Figs.~\ref{fig:spectral-norm-detailed}, \ref{fig:row-norm-detailed}, and \ref{fig:lambda-max-detailed}, show the evolution of $\norm{\W_t}_2^2$, $\norm{\g_t}_\infty$, and $\lambda_{\max}(\PP_t\,\C_t\,\PP_t)$, respectively. Each row depicts the quantities mentioned above computed on the weights of a transformer block of a 500M model trained with either of Muon and \algname{}. We remark that for all layers and weight types considered the spectral norm and maximal row norm are moving in lock step, while $\lambda_{\max}(\PP_t\,\C_t\,\PP_t)$ remains bounded.

Moreover, we provide more detailed results on the spectral distribution of the layers in the 16th transformer block in Figs.~\ref{fig:spectral-analysis-detailed}, \ref{fig:spectral-histogram-detailed}, and \ref{fig:spectral-distribution-lineplot}.

\begin{figure}
    \centering
    \includegraphics[width=\linewidth]{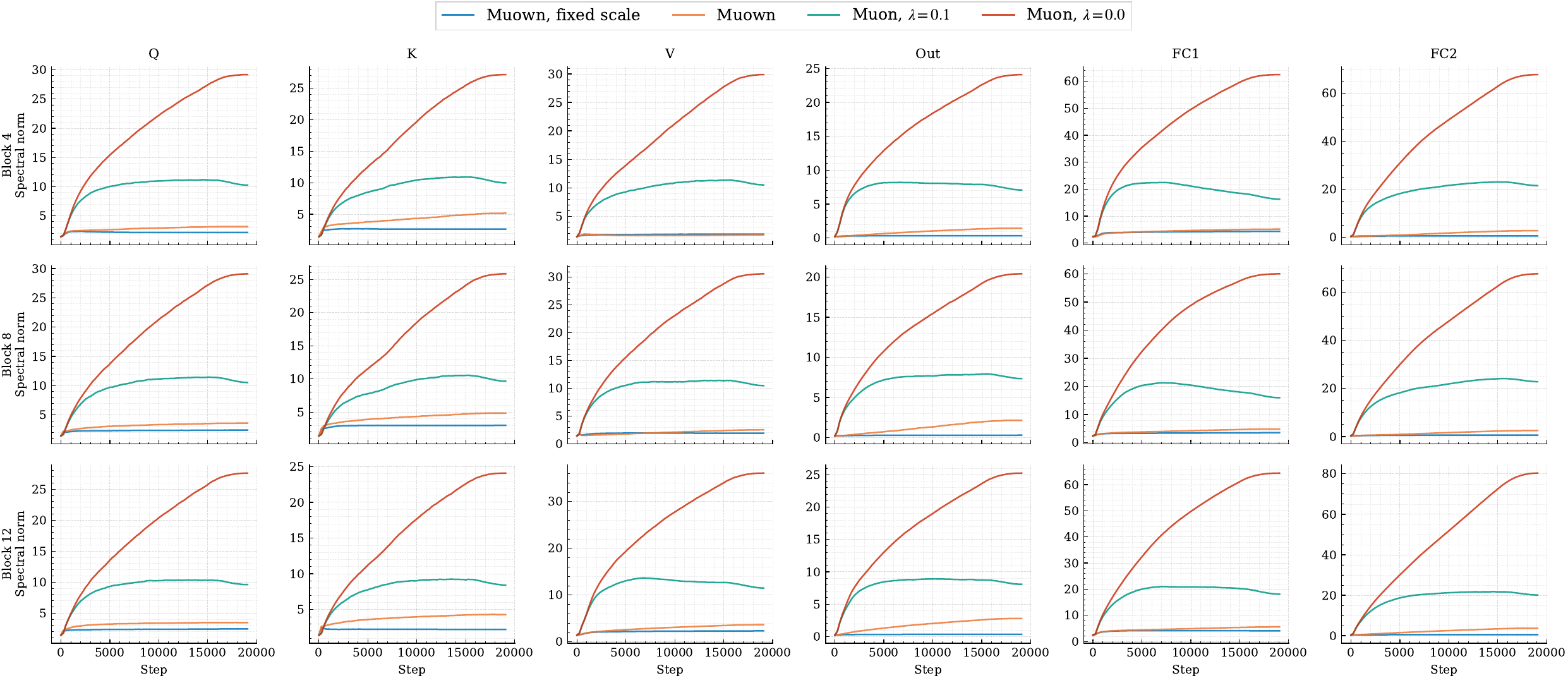}
    \caption{$\opnorm{\W_t}$ for the layers in the 4th, 8th, and 12th transformer block of a 500M model.}
    \label{fig:spectral-norm-detailed}
\end{figure}

\begin{figure}
    \centering
    \includegraphics[width=\linewidth]{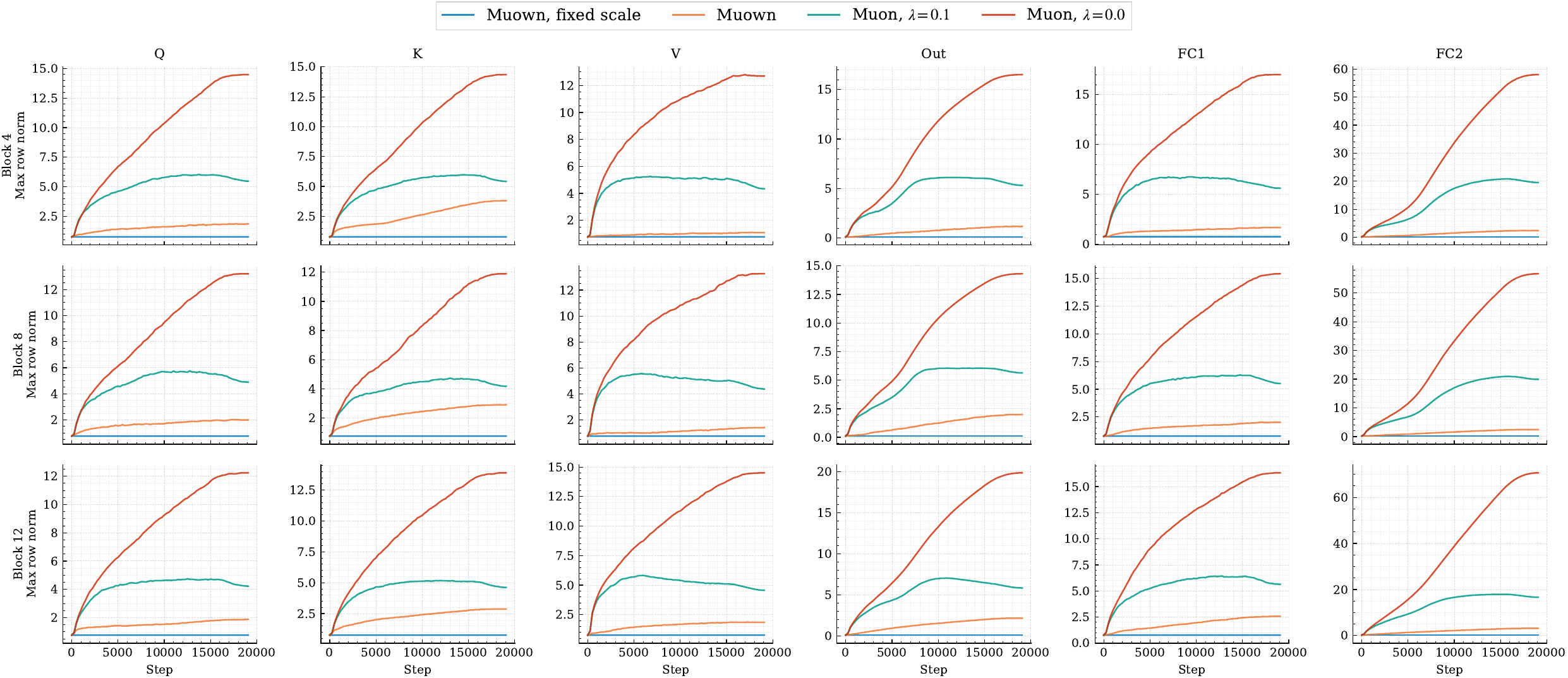}
    \caption{$\norm{\g_t}_\infty$ for the layers in the 4th, 8th, and 12th transformer block of a 500M model.}
    \label{fig:row-norm-detailed}
\end{figure}

\begin{figure}
    \centering
    \includegraphics[width=\linewidth]{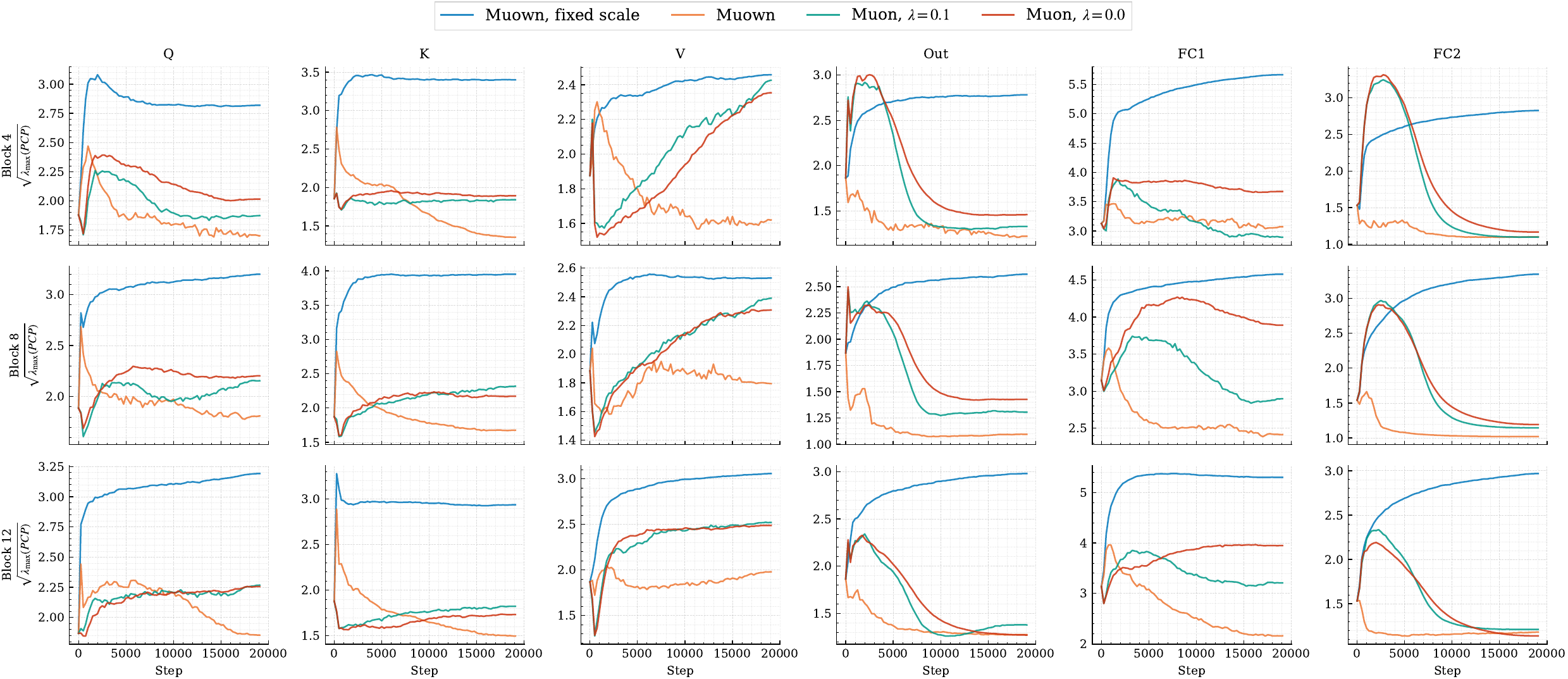}
    \caption{$\sqrt{\lambda_{\max}(\PP_t\,\C_t\,\PP_t)}$ for the layers in the 4th, 8th, and 12th transformer block of a 500M model.}
    \label{fig:lambda-max-detailed}
\end{figure}

\begin{figure}[t]
    \centering
    \begin{subfigure}[t]{0.245\textwidth}
        \centering
        \includegraphics[width=\textwidth]{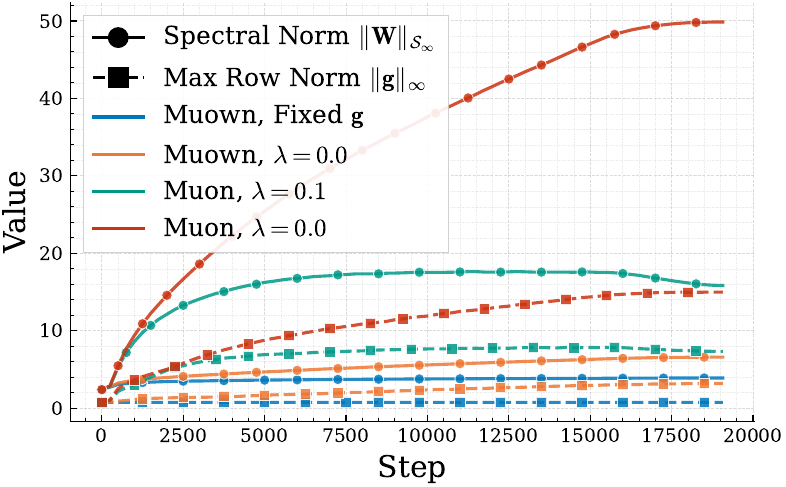}
        \caption{MLP Layer 1}
        \label{fig:spectral-analysis-detailed-mlp1}
    \end{subfigure}\hfill
    \begin{subfigure}[t]{0.245\textwidth}
        \centering
        \includegraphics[width=\textwidth]{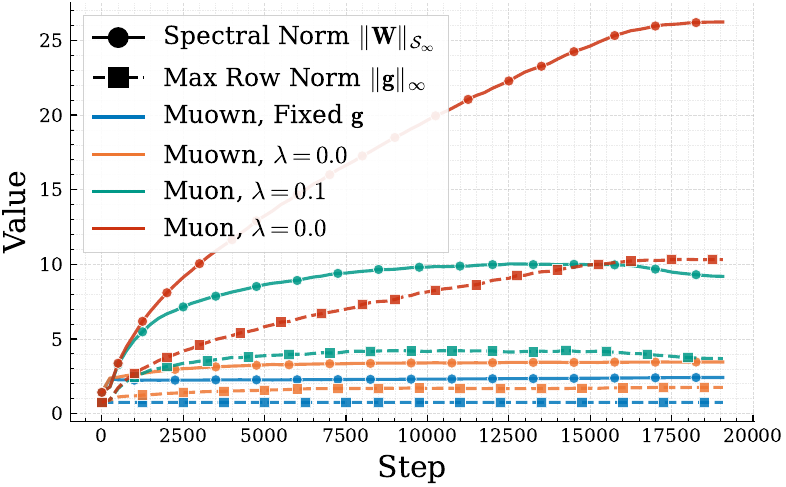}
        \caption{Query Projection}
    \end{subfigure}\hfill
    \begin{subfigure}[t]{0.245\textwidth}
        \centering
        \includegraphics[width=\textwidth]{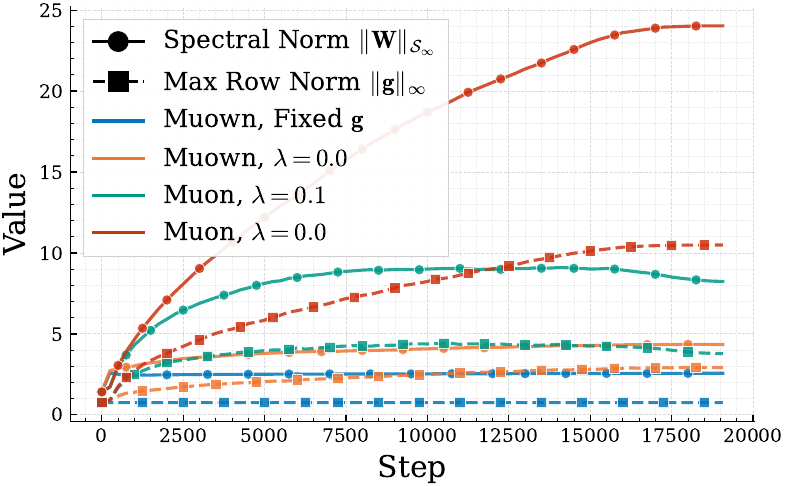}
        \caption{Key Projection}
    \end{subfigure}
    \begin{subfigure}[t]{0.245\textwidth}
        \centering
        \includegraphics[width=\textwidth]{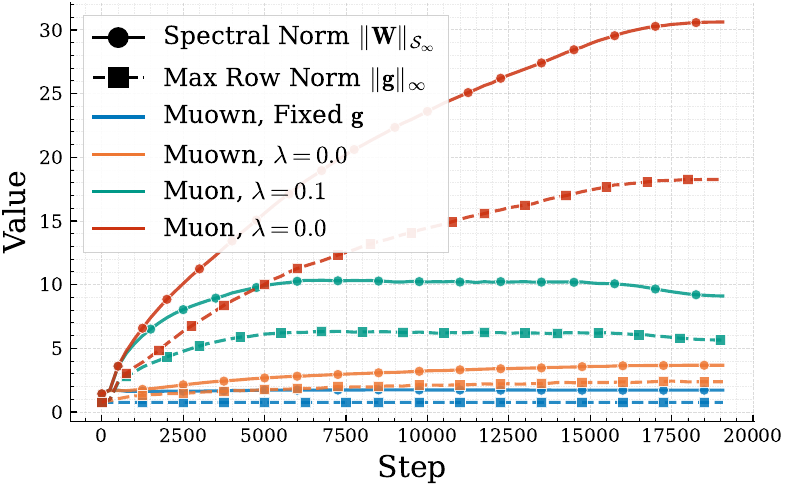}
        \caption{Value Projection}
    \end{subfigure}

    \caption{Evolution of spectral norm $\opnorm{\W_t}$ and maximum row norm $\|\g_t\|_{\infty}$ over the course of training for linear layers of the 16th transformer block of a 500M model. }
    \label{fig:spectral-analysis-detailed}
\end{figure}

\begin{figure}[t]
    \centering
    \begin{subfigure}[t]{0.245\textwidth}
        \centering
        \includegraphics[width=\textwidth]{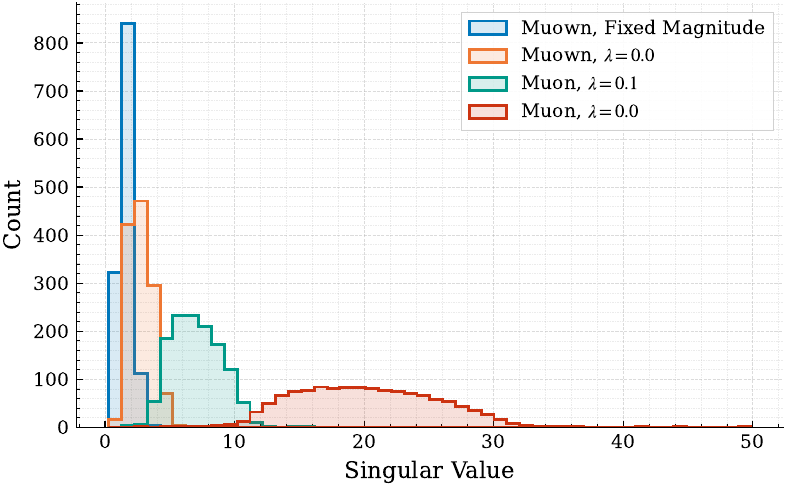}
        \caption{MLP Layer 1}
    
    \end{subfigure}\hfill
    \begin{subfigure}[t]{0.245\textwidth}
        \centering
        \includegraphics[width=\textwidth]{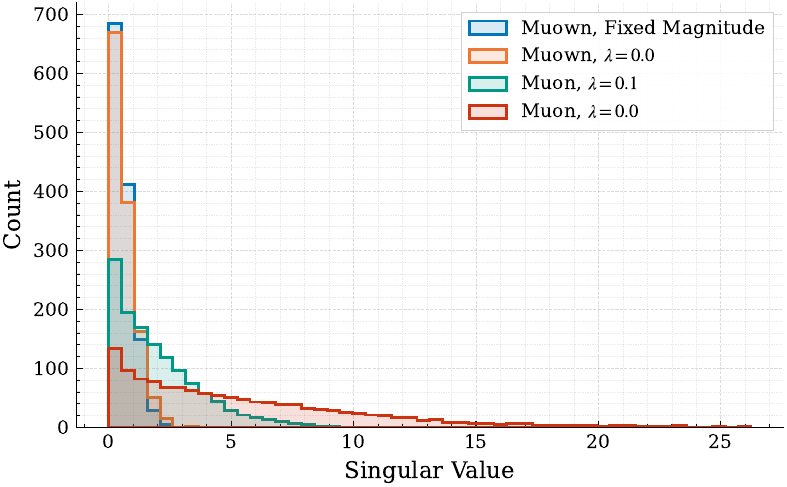}
        \caption{Query Projection}
    \end{subfigure}\hfill
    \begin{subfigure}[t]{0.245\textwidth}
        \centering
        \includegraphics[width=\textwidth]{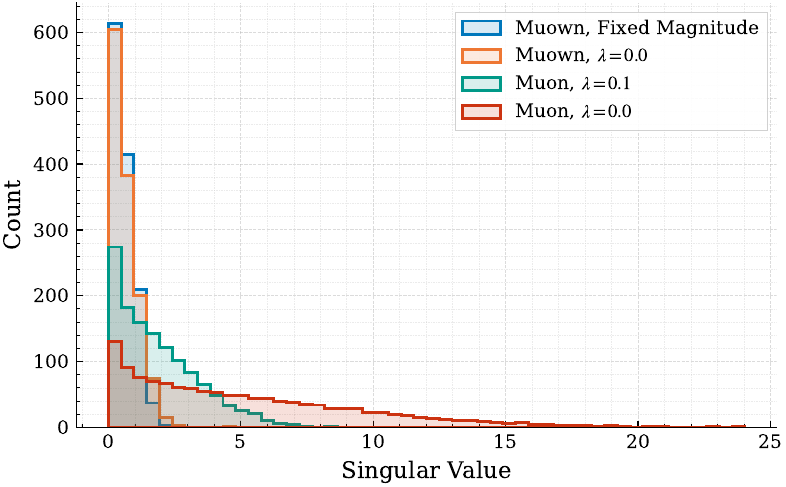}
        \caption{Key Projection}
    \end{subfigure}
    \begin{subfigure}[t]{0.245\textwidth}
        \centering
        \includegraphics[width=\textwidth]{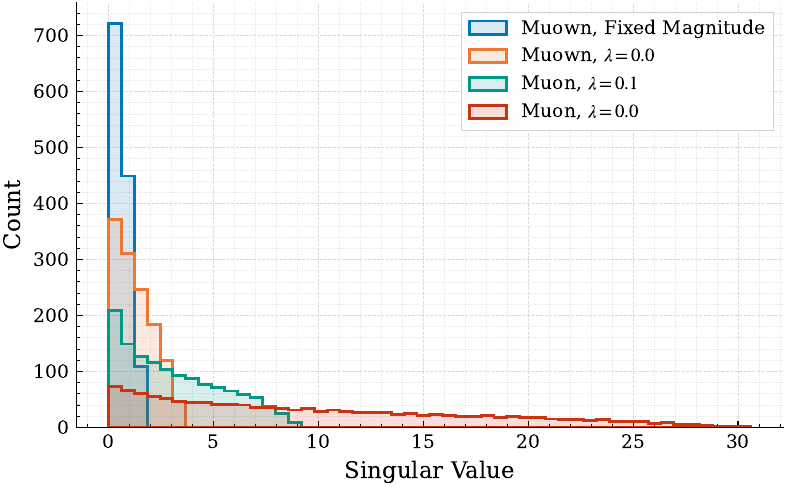}
        \caption{Value Projection}
    \end{subfigure}
        
    \caption{Histogram of singular values for linear layers of the 16th transformer block of a 500M model.}
    \label{fig:spectral-histogram-detailed}
\end{figure}

\begin{figure}
    \centering
    \includegraphics[width=\linewidth]{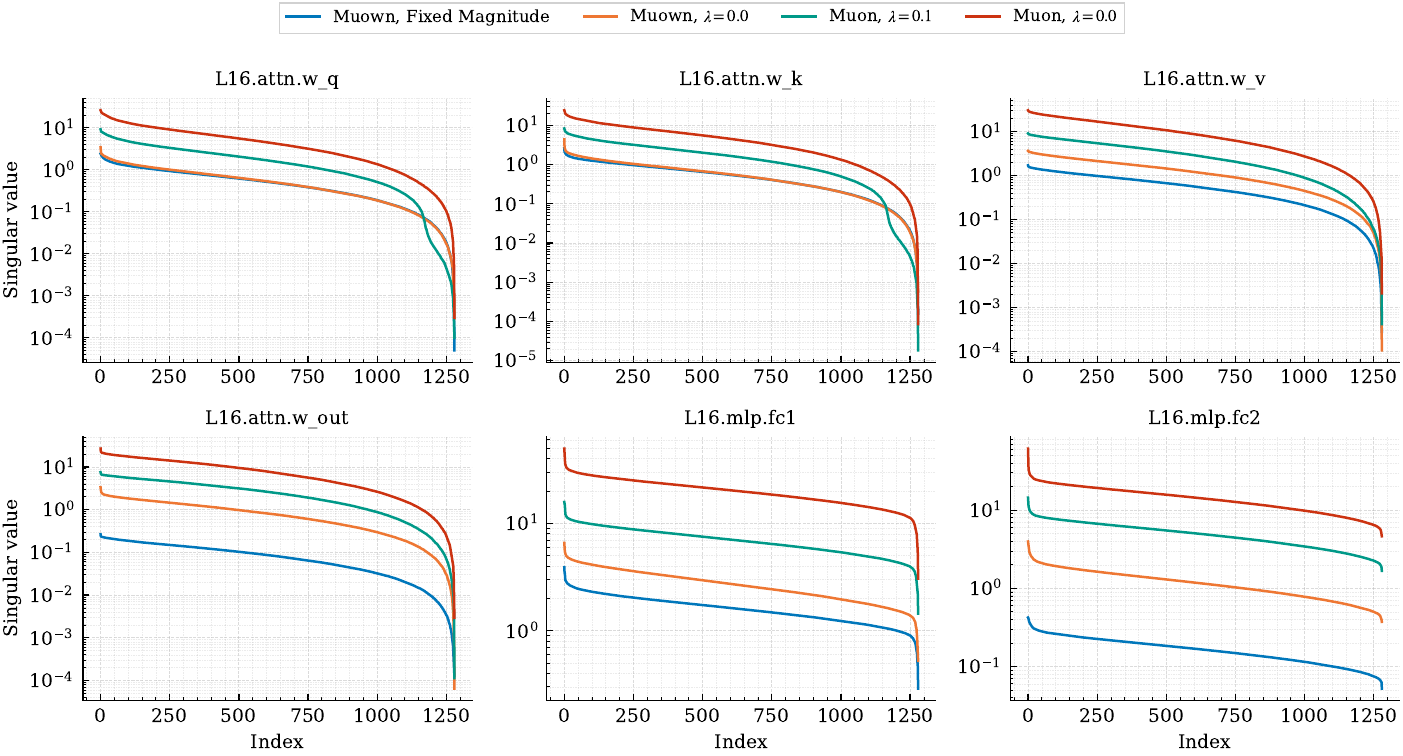}
    \caption{Singular value distribution for linear layers of the 16th transformer block of a 500M model. \texttt{attn} refers to the attention block with \texttt{w\_q}, \texttt{w\_k}, \texttt{w\_v}, \texttt{w\_out} denoting query, key, value, and output projections, respectively. \texttt{fc1} and \texttt{fc2} refer to the first and second MLP layers of the block.}
    \label{fig:spectral-distribution-lineplot}
\end{figure}

\subsection{Magnitude Learning}
\label{appdx:magnitude-learning}

In Section~\ref{sec:motivation}, we identify $\|\g_t\|_\infty$ as the empirical driver of $\opnorm{\W_t}$ drift under Muon and remark that freezing $\g_t = \g_1$ at initialization eliminates the systematic growth in spectral norm. Nevertheless, it remains open how effective such a strategy is in terms of perplexity. As Table~\ref{tab:magnitude-optimizer} demonstrates empirically, suppressing row-norm growth alone already accounts for a significant part of the improvement over Muon: \emph{Fixed}, which freezes row magnitudes at initialization, outperforms Muon at the smaller learning rates and prevents the divergence at $\eta = 4 \times 10^{-3}$ (Table~\ref{tab:magnitude-optimizer}), without learning the magnitudes. Training $\g$ on top, with either of SignSGD+Momentum (Signum) or Adam, refines the result further, with the ordering between Adam, Signum, and Fixed preserved across all three learning rates. While we find the \emph{Fixed} variant to perform slightly worse than the ones with trainable magnitudes, it can serve as a more lightweight alternative that does not require storage of the Adam states and further simplifies the update procedure.

\begin{table}
        \vspace{-10pt}
        \centering
        \setlength{\tabcolsep}{4pt}
        \caption{Perplexity for \algname{} on a 500M model using different magnitude optimizer choices.}
        \label{tab:magnitude-optimizer}
        \begin{tabular}{l ccc}
            \toprule
            $\eta$ & \multicolumn{3}{c}{\algname{}} \\
            \cmidrule(lr){2-4}
             &  Fixed & Signum & Adam \\
            \midrule
            $4 \times 10^{-3}$ & 14.22 & 14.21 & 14.11 \\
            $2 \times 10^{-3}$ & 14.21 & 14.18 & 14.11 \\
            $1 \times 10^{-3}$ & 14.30 & 14.27 & 14.24 \\
            \bottomrule
        \end{tabular}%
\end{table}

\subsection{Effective Rank of Minibatch Gradients}
\label{appdx:effective-rank}

Combined with Fig.~\ref{fig:gradient-conditioning} in the main text, Fig.~\ref{fig:gradient-conditioning-remaining} covers all six linear weight types of the transformer block (queries, keys, values, output projection, and both MLP layers). The same pattern holds across types: \algname{} sustains a higher normalized effective rank than Muon throughout training on every weight type considered, with the gap opening early in training and persisting for the entire run.

\begin{figure}[t]
        \centering
    \begin{subfigure}[t]{0.33\textwidth}
        \includegraphics[width=\textwidth]{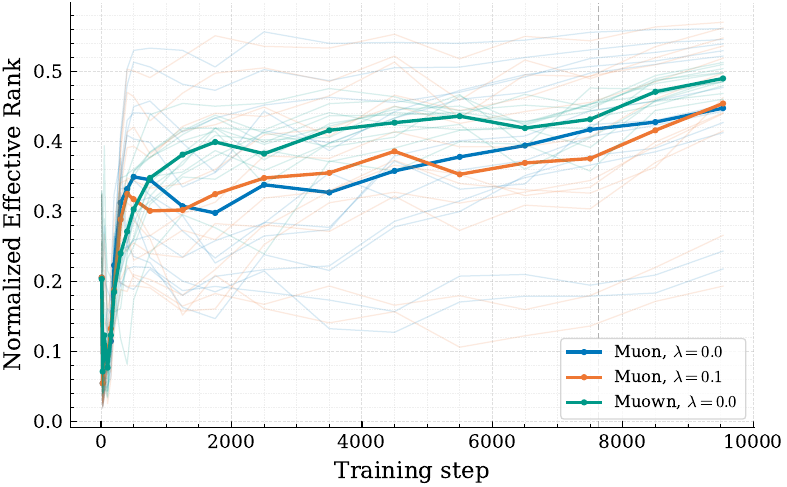}
        \caption{Key}
        \label{fig:eff-rank-w_k}
    \end{subfigure}
    \begin{subfigure}[t]{0.33\textwidth}
         \includegraphics[width=\textwidth]{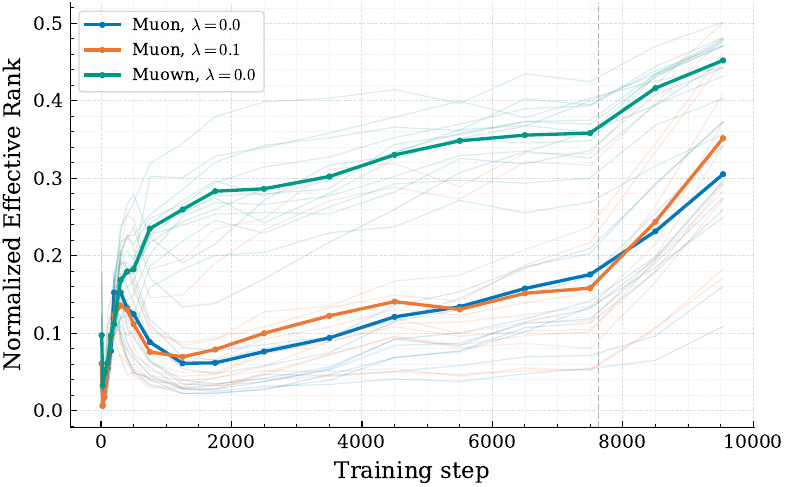}
         \caption{Value}
         \label{fig:eff-rank-w_v}
    \end{subfigure}\hfill
    \begin{subfigure}[t]{0.33\textwidth}
         \includegraphics[width=\textwidth]{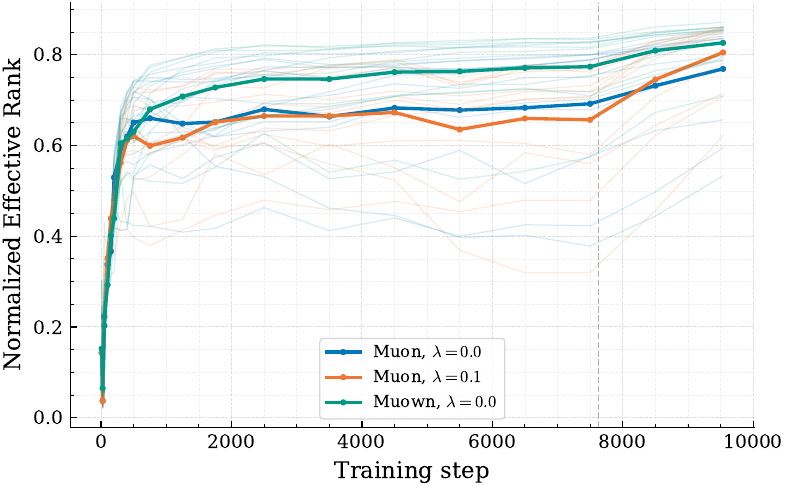}
         \caption{MLP Layer 2}
         \label{fig:eff-rank-fc2}
    \end{subfigure}\hfill

    \caption{Normalized effective rank \citep{roy_effective_2007} of minibatch gradients for a 124M model. At a given training step, we sample 8 gradients and average their effective ranks.}
    \label{fig:gradient-conditioning-remaining}
\end{figure}

\section{Proofs}
\label{appdx:proofs}

\subsection{Proposition~\texorpdfstring{\ref{prop:spectral-norm-decomposition}}{1}}
\label{ref:proof-spectral-norm-decomposition}

\begin{proof}
Since each row of $\W_t$ is nonzero, the vector $\g_t=\RowNorm{\W_t}$ is well-defined and has strictly positive entries. Hence
$$
\D_t=\mathrm{Diag}\!\left(1/\g_t\right)\W_t
$$
is well-defined, and for each row $i$,
$$
\dd_{t,i} = \frac{\ww_{t,i}}{\norm{\ww_{t,i}}_2},
$$
so $\norm{\dd_{t,i}}_2=1$. Therefore,
$$
\W_t=\mathrm{Diag}(\g_t)\D_t.
$$

Using this factorization, we compute
$$
\W_t\W_t^\top
=
\mathrm{Diag}(\g_t)\D_t\D_t^\top\mathrm{Diag}(\g_t)
=
\mathrm{Diag}(\g_t)\,\C_t\,\mathrm{Diag}(\g_t).
$$
By definition of the spectral norm,
$$
\opnorm{\W_t}^2
=
\lambda_{\max}(\W_t\W_t^\top),
$$
and therefore
$$
\opnorm{\W_t}^2
=
\lambda_{\max}\!\bigl(\mathrm{Diag}(\g_t)\,\C_t\,\mathrm{Diag}(\g_t)\bigr).
$$

It remains to factor out the maximal row magnitude. Since $\g_t=\RowNorm{\W_t}$, its entries are nonnegative, so $|\g_t|=\g_t$. By the definition of $\p_t$,
$$
\g_t=\norm{\g_t}_\infty\,\p_t,
\qquad\text{hence}\qquad
\mathrm{Diag}(\g_t)=\norm{\g_t}_\infty\,\PP_t.
$$
Substituting this into the previous identity yields
$$
\mathrm{Diag}(\g_t)\,\C_t\,\mathrm{Diag}(\g_t)
=
\norm{\g_t}_\infty^2\,\PP_t\,\C_t\,\PP_t.
$$
Since scaling a matrix by a positive scalar scales all eigenvalues by the same scalar, we obtain
$$
\lambda_{\max}\!\bigl(\mathrm{Diag}(\g_t)\,\C_t\,\mathrm{Diag}(\g_t)\bigr)
=
\norm{\g_t}_\infty^2\,\lambda_{\max}\!\bigl(\PP_t\,\C_t\,\PP_t\bigr).
$$
Recalling that the LHS of the equation above is equal to $\opnorm{\W_t}^2$ proves the claim.
\end{proof}

\subsection{Deterministic Convergence (Theorem~\texorpdfstring{\ref{thm:det_convergence}}{1})}
\label{sec:app.conv_analysis}

{
\begin{algorithm}[t]
    \caption{\algname{} (\signsgd\ Version)}\label{alg:muown_signgd}
    {\linespread{1.3}\selectfont
    \begin{algorithmic}
    \Require Initial point $\W_1 \in \mathbb{R}^{m\times n}$ with nonzero rows; stepsizes $\eta, \gamma>0$; momentum $\beta_1\in[0,1)$
    \State Initialize $\Rbf_1 \gets \W_1, \g_1 \gets \RowNorm{\W_1}, \M_0 \gets \nabla_\Rbf \mathcal L(\W_1, \xi_1)$, and $\m_0 \gets \nabla_\g \mathcal L (\W_1, \xi_1)$.
    \For{$t=1,2,\dots$}
        \State \textcolor{cyan}{\texttt{\# Update $\Rt$ with \muon}}
        \State $\Mt \gets \beta_1 \Mtm + \nabla_\Rbf \mathcal L(\Wt, \xi_t)$
        \State $\Ot \gets \arg\min_{\opnorm{\Obf} \leq 1} \langle \Mt, \Obf \rangle$\Comment{Choose arbitrary for set-valued $\arg\min$.}
        \State $\Rtp \gets \Rt + \eta \Ot$
        \State
        \State \textcolor{cyan}{\texttt{\# Update $\gt$ with \signsgd}}
        \State $\mt \gets \beta_1 \mtm + \nabla_\g \mathcal L\pare{\Wt, \xi_t}$
        \State $\gtp \gets \gt - \gamma \operatorname{sgn}\pare{\mt}$
        \State
        \State \textcolor{cyan}{\texttt{\# Update $\Wt$}}
        \State $\Wtp \gets \W\pare{\gtp, \Rtp}$
    \EndFor
    \end{algorithmic}
    }
\end{algorithm}

\renewcommand{\opnorm}[1]{\norm{#1}_{S_\infty}}
In this section we will provide the convergence analysis of Algorithm \ref{alg:muown_signgd}. Remember that 
\begin{equation*}
    \Ft \colon \mathbb{R}^m \times \rowPos \to \mathbb{R},\ \, \Ft(\g,\Rbf) := \mathcal L\pare{\W\pare{\g, \Rbf}} = \mathcal L\left(\Diag\left(\frac{\g}{\RowNorm{\Rbf}}\right)\, \Rbf\right).
\end{equation*}

\begin{theorem}[Deterministic Convergence of \algname{}]\label{thm:det_convergence}
    Assume $\Ft$ is $L$-smooth along the trajectory and lower bounded by $\Ft_\star$. Denote $\Delta_1 \geq \Ft(\g_1, \Rbf_1) - \Ft_\star$ and assume the deterministic setting, i.e., access to the true gradient $\nabla \mathcal L(\W)$. 
    Then, with $\beta_1 = 0$ and $\eta = \gamma = \sqrt{\frac{\Delta_1}{LT}}$, the iterates of Algorithm \ref{alg:muown_signgd} satisfy
    \begin{equation*}
        \frac 1 T \sum_{t=1}^T \norm{\nabla \Ft(\g_t,\Rbf_t)}_* \leq 4\sqrt{\frac{L\Delta_1}{T}}.
    \end{equation*}
    This matches the optimal $\mathcal O\pare{\sqrt{L\Delta_1/T}}$ non-convex rate (up to constants).
\end{theorem}

We will derive this result in steps. Straightforward calculations show that the gradient of $\Ft$ with respect to the canonical inner product is given by
\begin{align*}
    \nabla_\g \Ft\pare{\g, \Rbf} &= \mathrm{diag}\pare{\nabla \mathcal L\pare{\W} \D^\top},\\
    \nabla_\Rbf \Ft\pare{\g, \Rbf} &= \Diag\pare{\g/\RowNorm{\Rbf}}\pare{\nabla \mathcal L\pare{\W} - \Diag\pare{\nabla_\g \Ft\pare{\g,\Rbf}} \D}.
\end{align*}
where $\D = \Diag(1/\RowNorm{\Rbf})\Rbf$ is the row-normalization of $\Rbf$ and $\W = \Diag\pare{\g} \D$ is the original parameterization. In particular we note that, after re-parameterization, Muown updates $\Rbf$ and $\g$ with standard, albeit different, first-order updates. In particular Algorithm \ref{alg:muown_signgd} corresponds exactly to Muown when replacing Adam with its proxy \signsgd. In a first step we show that, under a smoothness assumption on $\Ft$, the mixed update of Algorithm \ref{alg:muown_signgd} can be analyzed separately and combined without issues.

\begin{lemma}\label{lem:split_algs}
    Let $\g, \Delta \g \in \mathbb{R}^m, \Rbf\in \rowPos, \Delta \Rbf \in \mathbb{R}^{m \times n}$ and $\gamma, \eta > 0$. Furthermore define $\g^+ := \g + \gamma \Delta \g, \Rbf^+ := \Rbf + \eta \Delta \Rbf$ and assume $\Rbf^+ \in \rowPos$ and that $\Ft$ is $L$-smooth between $(\g, \Rbf)$ and $(\g^+, \Rbf^+)$. Then
    \begin{align*}
        &\ \Ft\pare{\g^+, \Rbf^+} - \Ft\pare{\g, \Rbf} \\
        \leq &\  
        \underbrace{\gamma \langle \nabla_\g \Ft\pare{\g,\Rbf}, \Delta \g\rangle + \frac{L \gamma^2}{2}\norm{\Delta \g}_\infty^2}_{\text{only depends on $\Delta \g$}}
        + \underbrace{\eta \innerF{\nabla_\Rbf \Ft\pare{\g,\Rbf}}{\Delta \Rbf} + \frac{L}{2} \eta^2\opnorm{\Delta \Rbf}^2}_{\text{only depends on $\Delta \Rbf$}}.
    \end{align*}
\end{lemma}
\begin{proof}
    By the $L$-smoothness of $\Ft$ we have
    \begin{align*}
        &\ \Ft\pare{\g^+, \Rbf^+} - \Ft\pare{\g, \Rbf} \\
        \leq &\ \innerF{\nabla \Ft\pare{\g, \Rbf}}{(\g^+, \Rbf^+) - (\g, \Rbf)} + \frac{L}{2}\norm{(\g^+, \Rbf^+) - (\g, \Rbf)}^2 \\
        = &\ \gamma \langle \nabla_\g \Ft\pare{\g,\Rbf}, \Delta \g\rangle + \eta \innerF{\nabla_\Rbf \Ft\pare{\g,\Rbf}}{\Delta \Rbf} + \frac{L}{2}\max\{\gamma^2\norm{\Delta \g}_\infty^2, \eta^2\opnorm{\Delta \Rbf}^2\} \\
        \leq &\  
        \gamma \langle \nabla_\g \Ft\pare{\g,\Rbf}, \Delta \g\rangle + \frac{L \gamma^2}{2}\norm{\Delta \g}_\infty^2
        + \eta \innerF{\nabla_\Rbf \Ft\pare{\g,\Rbf}}{\Delta \Rbf} + \frac{L}{2} \eta^2\opnorm{\Delta \Rbf}^2,
    \end{align*}
    where we used our choice of inner product and norm in the equality.
\end{proof}

Now we are ready to prove Theorem \ref{thm:det_convergence}.

\begin{proof}[Proof of Theorem \ref{thm:det_convergence}]
    Let $t \in [T]$. We apply Lemma \ref{lem:split_algs} with
    \begin{equation*}
        \Delta \g_t := -\operatorname{sgn}\pare{\nabla_\g \Ft\pare{\g_t,\Rbf_t}} \qquad  \text{and} \qquad
        \Delta \Rbf_t := \operatorname{argmin}_{\opnorm{\Obf} \leq 1}\langle \nabla_\Rbf \Ft\pare{\g_t,\Rbf_t}, \Obf \rangle.
    \end{equation*}
    to get
    \begin{align*}
        &\ \Ft\pare{\g_{t+1}, \Rbf_{t+1}} - \Ft\pare{\g_t, \Rbf_t} \\
        \leq &\  
        \underbrace{\gamma \langle \nabla_\g \Ft \pare{\gt, \Rt}, \Delta \g_t\rangle + \frac{L \gamma^2}{2}\norm{\Delta \g_t}_\infty^2}_{\text{only depends on $\Delta \g_t$}}
        + \underbrace{\eta \innerF{\nabla_\Rbf \Ft\pare{\gt,\Rt}}{\Delta \Rbf_t} + \frac{L}{2} \eta^2\opnorm{\Delta \Rbf_t}^2}_{\text{only depends on $\Delta \Rbf_t$}}
    \end{align*}
    For the $\Delta \g_t$-term note that $\langle \nabla_\g \Ft\pare{\gt,\Rt}, \Delta \g_t\rangle = -\norm{\nabla_\g \Ft\pare{\gt,\Rt}}_1$ and $\norm{\Delta \g_t}_\infty \leq 1$, thus
    \begin{equation*}
        \gamma \langle \nabla_\g \Ft\pare{\gt,\Rt}, \Delta \g_t\rangle + \frac{L \gamma^2}{2}\norm{\Delta \g_t}_\infty^2
        \leq
        -\gamma \norm{\nabla_\g \Ft\pare{\gt,\Rt}}_1 + \frac{L \gamma^2}{2}.
    \end{equation*}
    For the $\Delta \Rt$-term we have $\innerF{\nabla_\Rbf \Ft\pare{\gt,\Rt}}{\Delta \Rt} = -\trnorm{\nabla_\Rbf \Ft\pare{\gt,\Rt}}$, where $\trnorm \cdot$ is the Schatten-1 (trace) norm and $\opnorm{\Delta \Rt} \leq 1$, thus
    \begin{equation*}
        \eta \innerF{\nabla_\Rbf \Ft\pare{\gt,\Rt}}{\Delta \Rt} + \frac{L}{2} \eta^2\opnorm{\Delta \Rt}^2
        \leq
        -\eta \trnorm{\nabla_\Rbf \Ft\pare{\gt,\Rt}} + \frac{L}{2} \eta^2.
    \end{equation*}
    Combining these bounds gives
    \begin{equation*}
        \Ft\pare{\g_{t+1}, \Rbf_{t+1}} - \Ft\pare{\g_t, \Rbf_t} 
        \leq
        -\gamma \norm{\nabla_\g \Ft\pare{\gt,\Rt}}_1 -\eta \trnorm{\nabla_\Rbf \Ft\pare{\gt,\Rt}} + \frac{L}{2}(\eta^2 + \gamma^2).
    \end{equation*}
    Summing this inequality from $t=1$ to $T$ and using $\Ft\pare{\g_{T+1}, \Rbf_{T+1}} \geq \Ft_\star$ gives
    \begin{equation*}
        \sum_{t=1}^{T}\pare{\gamma \norm{\nabla_\g \Ft\pare{\gt,\Rt}}_1 + \eta \trnorm{\nabla_\Rbf \Ft\pare{\gt,\Rt}}}
        \leq \Delta_1 + \frac{L}{2}(\eta^2 + \gamma^2)T.
    \end{equation*}
    Define $R_T := \Delta_1 + \frac{L}{2}(\eta^2 + \gamma^2)T$. Since both terms in the sum are nonnegative, we can drop one term at a time to obtain
    \begin{equation*}
        \frac{1}{T}\sum_{t=1}^T \norm{\nabla_\g \Ft\pare{\gt,\Rt}}_1 \leq \frac{R_T}{\gamma T},
        \qquad
        \frac{1}{T}\sum_{t=1}^T \trnorm{\nabla_\Rbf \Ft\pare{\gt,\Rt}} \leq \frac{R_T}{\eta T}.
    \end{equation*}
    By our choice of product norm $\norm{(\g,\Rbf)} = \max\{\norm{\g}_\infty,\opnorm{\Rbf}\}$, the corresponding dual norm satisfies
    \begin{equation*}
        \norm{\nabla \Ft\pare{\gt,\Rt}}_* = \norm{\nabla_\g \Ft\pare{\gt,\Rt}}_1 + \trnorm{\nabla_\Rbf \Ft\pare{\gt,\Rt}}.
    \end{equation*}
    Combining the two bounds yields
    \begin{equation*}
        \frac{1}{T}\sum_{t=1}^T \norm{\nabla \Ft\pare{\gt,\Rt}}_* \leq \frac{R_T}{\gamma T} + \frac{R_T}{\eta T}.
    \end{equation*}
    
\end{proof}

\subsection{Stochastic Convergence (Theorem~\texorpdfstring{\ref{thm:stoch_convergence}}{2})}
\label{sec:app.conv_analysis.stochastic}

For the stochastic convergence result, we will use the following standard bounded variance assumption on the stochastic gradients along the trajectory of Algorithm \ref{alg:muown_signgd}.
\begin{assumption}\label{assum:bounded_var}
    The stochastic gradients along the trajectory are unbiased, i.e., 
    \begin{equation*}
        \E[\nabla_{\g} \Ft(\gt, \Rt, \xi_t)] = \nabla_{\g} \Ft(\gt, \Rt), \qquad 
        \E[\nabla_{\Rbf} \Ft(\gt, \Rt, \xi_t)] = \nabla_{\Rbf} \Ft(\gt, \Rt),
    \end{equation*}
    and have bounded variance
    \begin{equation*}
        \E \left[\norm{\nabla_{\g} \Ft(\gt, \Rt, \xi_t) - \nabla_{\g} \Ft(\gt, \Rt)}_1^2\right] \leq \sigg^2,  \ 
        \E \left[\trnorm{\nabla_{\Rbf} \Ft(\gt, \Rt, \xi_t) - \nabla_{\Rbf} \Ft(\gt, \Rt)}^2\right] \leq \sigR^2.
    \end{equation*}
\end{assumption}

Now we are ready to prove the Theorem.

\begin{proof}

As in the deterministic case, we apply Lemma \ref{lem:split_algs} to get
\begin{align}\begin{split}\label{eq:descent_ineq}
    &\ \Ft\pare{\gtp, \Rtp} - \Ft\pare{\gt, \Rt} \\
    \leq &\  
    \gamma \langle \nabla_\g \Ft\pare{\gt,\Rt}, - \sgn\pare{\mt}\rangle
    + \eta \innerF{\nabla_\Rbf \Ft\pare{\gt,\Rt}}{\Ot} + \frac{L}{2}(\gamma^2 + \eta^2).
\end{split}\end{align}
We start by bounding the first term on the right-hand side through the classical decomposition of the inner product \citep[see e.g.][]{MomentumImprovesNormalized2020cutkosky} to get
\begin{align*}
    &\ \left\langle \nabla_\g \Ft\pare{\gt,\Rt}, - \sgn\pare{\mt}\right\rangle\\
    =&\  \left\langle \nabla_\g \Ft\pare{\gt,\Rt} - (1-\beta_1)\mt, - \sgn\pare{\mt}\right\rangle 
    + \left\langle (1-\beta_1)\mt, - \sgn\pare{\mt}\right\rangle\\
    \leq &\ \norm{\nabla_\g \Ft\pare{\gt,\Rt} - (1-\beta_1)\mt}_1 \norm{-\sgn\pare{\mt}}_\infty - (1-\beta_1)\norm{\mt}_1 \\
    \leq &\ \norm{\nabla_\g \Ft\pare{\gt,\Rt} - (1-\beta_1)\mt}_1 - (1-\beta_1)\norm{\mt}_1\\
    \leq &\ 2\norm{\nabla_\g \Ft\pare{\gt,\Rt} - (1-\beta_1)\mt}_1 - \norm{\gradgt}_1,
\end{align*}
where we applied the Hölder inequality in the second, and the triangle inequality in the fourth step. Further, when denoting $\mubf_t := \gradgt - (1-\beta_1)\mt, \eps_t := \nabla_\g \Ft(\gt, \Rt, \xi_t) - \gradgt$, 
\begin{equation*}
    \mubf_t = 
    \beta_1^{t-1} \mubf_1 
    - (1-\beta_1) \sum_{\tau=1}^{t-1} \beta_1^{\tau - 1} \eps_{t- \tau + 1}
    + \sum_{\tau=1}^{t-1} \beta_1^{\tau} \pare{\nabla_\g \Ft(\g_{t - \tau}, \Rbf_{t-\tau}) - \nabla_\g \Ft(\g_{t + 1 - \tau}, \Rbf_{t + 1 - \tau})}
\end{equation*}
and hence
\begin{align}\label{eq:mu_decomposition}\begin{split}
    \norm{\mubf_t}_1 \leq
    &\ \beta_1^{t-1} \norm{\mubf_1}_1 
    + (1-\beta_1) \norm{\sum_{\tau=1}^{t-1} \beta_1^{\tau - 1} \eps_{t- \tau + 1}}_1\\
    &+ \sum_{\tau=1}^{t-1} \beta_1^{\tau} \norm{\nabla_\g \Ft(\g_{t - \tau}, \Rbf_{t-\tau}) - \nabla_\g \Ft(\g_{t + 1 - \tau}, \Rbf_{t + 1 - \tau})}_1\\
    \leq &\ \beta_1^{t-1} \norm{\mubf_1}_1 
    + (1-\beta_1) \norm{\sum_{\tau=1}^{t-1} \beta_1^{\tau - 1} \eps_{t- \tau + 1}}_1 
    + \gamma L \sum_{\tau=1}^{t-1} \beta_1^{\tau} \\
    \leq &\ \beta_1^{t-1} \norm{\mubf_1}_1 
    + (1-\beta_1) \norm{\sum_{\tau=1}^{t-1} \beta_1^{\tau - 1} \eps_{t- \tau + 1}}_1 
    + \frac{\gamma L}{1-\beta_1}
\end{split}\end{align}
Following existing analyses of Muon \citep{ConvergenceAnalysisMuon2025shen}, we next use equivalence of norms in finite dimensional spaces to derive
\begin{align*}
    \E\left[\norm{\sum_{\tau=1}^{t-1} \beta_1^{\tau - 1} \eps_{t- \tau + 1}}_1\right]
    &\leq \zeta_\g \E\left[\norm{\sum_{\tau=1}^{t-1} \beta_1^{\tau - 1} \eps_{t- \tau + 1}}_2\right]
    \leq \zeta_\g \sqrt{\E\left[\norm{\sum_{\tau=1}^{t-1} \beta_1^{\tau - 1} \eps_{t- \tau + 1}}_2^2\right]}\\
    &= \zeta_\g \sqrt{\sum_{\tau=1}^{t-1} \beta_1^{2(\tau-1)} \E\left[\norm{\eps_{t- \tau + 1}}_2^2\right]} 
    \leq \frac{\zeta_\g \sigg}{\sqrt{1-\beta_1^2}},
\end{align*}
where we used the $L^2$-orthogonality of the error terms in the second, and $\norm \cdot_2 \leq \norm \cdot_1$ in the last step. Plugging into \eqref{eq:mu_decomposition} and summing yields
\begin{align*}
    \sum_{t=1}^T \E\left[\norm{\mubf_t}_1\right] \leq
    &\ \sum_{t=1}^T \beta_1^{t-1} \E [\norm{\mubf_1}_1] 
    + (1-\beta_1)\sum_{t=1}^T \E\left[\norm{\sum_{\tau=1}^{t-1} \beta_1^{\tau - 1} \eps_{t- \tau + 1}}_1\right] 
    + \frac{\gamma L}{1-\beta_1}T\\
    \leq &\ 
    \frac{1}{1-\beta_1} \E [\norm{\mubf_1}_1] 
    + \frac{\zeta_\g (1-\beta_1) \sigma}{\sqrt{1-\beta_1^2}}T
    + \frac{\gamma L}{1-\beta_1}T\\
    \leq &\ 
    \frac\sigg{1-\beta_1}
    + \zeta_\g \sigg \sqrt{1-\beta_1} T
    + \frac{\gamma L}{1-\beta_1}T,
\end{align*}
where we used our choice of $\m_1$ in the last step. Applying the same arguments to the $\Rbf$-term in \eqref{eq:descent_ineq} yields
\begin{align*}
    \sum_{t=1}^T \left\langle \nabla_\Rbf \Ft\pare{\gt,\Rt}, \Ot\right\rangle
    &\leq 
    \frac {2\sigR} {1-\beta_1}  
    + 2\zeta_\Rbf \sigR \sqrt{1-\beta_1} T
    + \frac{2\eta L}{1-\beta_1}T
    - \sum_{t=1}^T \trnorm{\gradRt}.
\end{align*}
Summing \eqref{eq:descent_ineq} from $t=1$ to $T$ and using $\Ft\pare{\g_{T+1}, \Rbf_{T+1}} \geq \Ft_\star$ gives
\begin{align*}
    &\ \sum_{t=1}^T \E\left[\gamma \norm{\gradgt}_1 + \eta \trnorm{\gradRt}\right] \\
    \leq &\ 
    \Delta_1
    + L \pare{\frac{5\gamma^2}{2(1-\beta_1)} + \frac{5\eta^2}{2(1-\beta_1)}}T
    + 2 (\zeta_\g \sigg + \zeta_\Rbf \sigR)\eta \sqrt{1-\beta_1}  T
    + \frac{2(\sigg + \sigR)\eta}{1-\beta_1}.
\end{align*}
By our choices of $\gamma = \eta$ and definition $\hat \sigma = \zeta_\g \sigg + \zeta_\Rbf \sigR$, we can simplify
\begin{equation}
    \frac{1}{T}\sum_{t=1}^T \E\left[\norm{\nabla \Ft\pare{\gt,\Rt}}_*\right]
    \leq 
    \frac{\Delta_1}{\eta T}
    + \frac{5 \eta L}{(1- \beta_1)}
    + 2 \sqrt{1-\beta_1}\hat \sigma 
    + 2\frac{\sigg + \sigR}{(1-\beta_1) T},
\end{equation}
where we used the fact that $\norm{\nabla \Ft \pare{\gt, \Rt}}_* = \norm{\gradgt}_1 + \trnorm{\gradRt}$. Plugging our choice $\gamma = \eta = \sqrt{\frac{\Delta_1 (1-\beta_1)}{L T}}$ in the above inequality yields
\begin{equation*}
    \frac{1}{T}\sum_{t=1}^T \E\left[\norm{\nabla \Ft\pare{\gt,\Rt}}_*\right]
    \leq 
    6\sqrt{\frac{\Delta_1 L}{(1-\beta_1) T}}
    + 2 \sqrt{1-\beta_1} \hat \sigma
    + 2\frac{\sigg + \sigR}{(1-\beta_1) T}.
\end{equation*}
Further using the choice $\beta_1 = 1 - \min\left\{1, \max\left\{T^{-\nicefrac 2 3}, \sqrt{\frac{\Delta_1 L}{\hat \sigma^2 T}}\right\}\right\}$ gives
\begin{align*}
    \sqrt{\frac{\Delta_1 L}{(1-\beta_1) T}}
    &\leq \sqrt{\frac{\Delta_1 L}{T}} + \pare{\frac{\Delta_1 L \hat \sigma^2}{T}}^{\nicefrac 1 4}\\
    \sqrt{1-\beta_1} \hat \sigma
    &\leq \frac{\hat \sigma}{T^{\nicefrac 1 3}} + \pare{\frac{\Delta_1 L \hat \sigma^2}{T}}^{\nicefrac 1 4}\\
    \frac{ \sigg + \sigR}{(1-\beta_1) T}
    &\leq \frac{\sigg + \sigR}{T^{\nicefrac 1 3}}
    \leq \frac{\hat \sigma}{T^{\nicefrac 1 3}}
\end{align*}
Combining these bounds yields 
\begin{equation*}
    \frac{1}{T}\sum_{t=1}^T \E\left[\norm{\nabla \Ft\pare{\gt,\Rt}}_*\right]
    \leq 
    6 \sqrt{\frac{\Delta_1 L}{T}}
    + 8 \pare{\frac{\Delta_1 L \hat \sigma^2}{T}}^{\nicefrac 1 4}
    + \frac{4 \hat \sigma}{T^{\nicefrac 1 3}}
\end{equation*}
and hence the result.
    
\end{proof}
}

\newpage
\section{Experimental Details}
\label{appdx:experimental-details}

\textbf{Further Remarks on Algorithm~\ref{alg:algname}.\ } We use learning rate scaling to match the Adam update RMS norm \citep{liu_muon_2025} for Muon as well as \algname{}. This corresponds to a learning rate adjustment by a factor of $0.2 \sqrt{\max(m, n)}$ for a layer of shape $m \times n$. Moreover, \algname{} as outlined in Algorithm~\ref{alg:algname} uses simplified Nesterov momentum \citep{MethodSolvingConvex1983nesterov, bengio_2013_advances}, matching the PyTorch implementation of Muon. 

To keep Algorithm~\ref{alg:algname} concise, we describe the use of decoupled weight decay with \algname{} separately here. When decoupled weight decay with coefficient $\lambda$ is enabled, we replace the last line of Algorithm~\ref{alg:algname} with $\W \gets \mathrm{Diag}(\frac{\g}{\rbf}) \Rbf- \eta_t\lambda \W$ and refresh $\g \gets \RowNorm{\W}$ to preserve the invariant $\g = \RowNorm{\W}$. This is only one of several natural placements of weight decay under the $(\g, \Rbf)$-parameterization. For instance, one could equally well decay only the direction component $\Rbf$. The variant adopted above stays closest to the canonical decoupled formulation \citep{loshchilov_decoupled_2019} on the effective weight $\W$ and is the one used throughout our experiments. A systematic study of other alternatives is left to future work.
 
\textbf{General Setup.\ } If not noted otherwise, we run all experiments on top of the implementation of \cite{ajroldi2024plainlm}. We make our code available on github.\footnote{\url{https://github.com/kcc-lion/muown}} We train a transformer with Rotational Positional Embeddings \citep{su_rope}, RMSNorm \citep{zhang_rmsnorm}, SwiGLU MLPs \citep{shazeer_glu_2020}, expansion ratio of 8/3, and weight tying. We employ a warmup-stable-decay schedule for 2\% of training as warmup and 20\% as cooldown. In all but the distributed benchmarking experiment in Table~\ref{tab:timing-overhead}, we make use of the PyTorch implementation of Muon. We justify this deviation below. We employ batch size of 524k tokens per step training, training for a total token budget of 20$\times$ tokens/parameter \citep{hoffmann_training_2022}, if not otherwise noted. Regarding sequence length and training budget we make the following size-dependent choices:
\begin{itemize}
    \item For models of the 124M class, we choose 12 layers with hidden dimension of 768 and 12 attention heads. We train with sequence length 1024 for 5B tokens and validate on 4M tokens. We remark that due to weight tying the actual number of parameters is 124M.
    \item For models with 500M parameters, we choose 22 layers with hidden dimension of 1280 and 20 attention heads. We train with sequence length 2048 for 10B tokens and validate on 10M tokens. 
    \item For models with 1B parameters, we choose 24 layers with hidden dimension of 1792 and 28 heads. We train for 20B tokens with sequence length 2048, validating on 10M tokens. 
    \item For models with 2.7B parameters, we choose 32 layers with dimension 2560 and 40 heads and train for 53B tokens, keeping everything else as in the 1B setup. 
\end{itemize}

All models are trained on the 100B token sample of FineWeb-Edu \citep{lozhkov2024fineweb-edu} and tokenized with the GPT-NeoX-20B tokenizer \citep{black_gpt-neox-20b_2022}, if not noted otherwise. Experiments are performed on either of NVIDIA GH200 and NVIDIA H100 GPUs and we estimate the total GPU hours spent as roughly $\approx$18k hours, including experimentation, development, and final runs.

\textbf{Spectral Analysis (Fig.~\ref{fig:sv_histograms}, \ref{fig:spectral-norm-detailed}, \ref{fig:row-norm-detailed}, \ref{fig:lambda-max-detailed}, \ref{fig:spectral-histogram-detailed}, \ref{fig:spectral-distribution-lineplot}). } For the spectral analysis plots, we adopt the 500M setup outlined above and tune Muon and \algname{} with learning rate $\eta \in \{1 \times 10^{-3}, 2 \times 10^{-3}, 4 \times 10^{-3}\}$ and weight decay $\lambda \in \{ 0, 0.1 \}$. We report results for the best runs in terms of validation loss, i.e. $\eta = 2 \times 10^{-3}$ with and without weight decay for Muon and $\eta = 4 \times 10^{-3}$ without weight decay for \algname{}. For the setup with fixed row magnitude, we re-use the best config from \algname{}.

\textbf{Learning Rate Ablation (Fig.~\ref{fig:lr-ablation}).} For the ablation study in Fig.~\ref{fig:lr-ablation}, we use a grid of $\eta \in \{2.5 \times 10^{-4}, 5 \times 10^{-4}, 7.5 \times 10^{-4}, 10^{-3}, 2 \times 10^{-3}, 4 \times 10^{-3}, 6 \times 10^{-3}, 8 \times 10^{-3}\}$ for all optimizers apart from Lion. For Muon's weight decay parameter, we use $\lambda=0.1$, which we found to be optimal across a grid of $\lambda \in \{0.01, 0.03, 0.1, 0.3\}$ (see Fig.~\ref{fig:weight-decay-ablation}). Similarly, based on the results of Fig.~\ref{fig:weight-decay-ablation}, we choose $\lambda=0.01$ for \algname{} when using weight decay, despite having only a marginal effect. For Lion, we use $\eta \in \{ 8 \times 10^{-5}, 10^{-4}, 2.5 \times 10^{-4}, 5 \times 10^{-4}, 7.5 \times 10^{-4}\}$, as Lion typically has a larger update norm than other optimizers, requiring a lower learning rate \citep{chen_symbolic_2023}. This matches our observation of divergence for learning rates larger than these. Moreover, \cite{chen_symbolic_2023} prescribe a larger value for weight decay in the range of 3-10$\times$ of the one used for AdamW. We opt for $5\times$ the optimal value of Muon, yielding $\lambda=0.5$ for Lion. We retain the default values of $(\beta_1, \beta_2) = (0.9, 0.99)$ for Lion. For SOAP \citep{vyas_soap_2024}, we use the default hyperparameters from the original implementation: $(\beta_1, \beta_2) = (0.9, 0.95)$, $\epsilon = 10^{-8}$, a preconditioner update frequency of 10 steps, and bias correction enabled. We set the maximum precondition dimension to 10{,}000, which excludes the embedding and head layers (vocabulary size 50{,}280) from preconditioning while covering all other layers. Weight decay for SOAP is set at $\lambda=10^{-4}$, following \cite{vyas_soap_2024}.

\textbf{2.7B Runs (Fig.~\ref{fig:3B-plot}, Table~\ref{tab:1B-results}).\ } We adopt the 2.7B setup above and sweep $\eta \in \{7.5\times 10^{-4}, 10^{-3}, 2\times 10^{-3}\}$, with $\lambda=0$ for \algname{} and $\lambda \in \{0, 0.1\}$ for Muon. The $\eta=2\times 10^{-3}$ Muon runs were aborted after 12h on 16 GPUs as the validation loss was diverging.

\textbf{Weight Decay Ablation (Fig.~\ref{fig:weight-decay-ablation}).} We adopt the 124M setup described above and train for 5B tokens ($2\times$ Chinchilla-optimal), sweeping a two-dimensional grid over the learning rate $\eta \in \{7.5 \times 10^{-4}, 10^{-3}, 2 \times 10^{-3}, 4 \times 10^{-3}, 6 \times 10^{-3}\}$ and weight decay $\lambda \in \{0, 0.01, 0.03, 0.1, 0.3\}$ for both Muon and \algname{}. We report perplexity at $2\times$ the Chinchilla-optimal token count (i.e. $5$B tokens).

\textbf{Learning Rate and Width Ablation (Fig.~\ref{fig:width-ablation}).} We use the 124M architecture outlined above and vary the width as 512, 768, 1280, resulting in models of size 67M, 124M, and 308M. We train each size for a Chinchilla-optimal amount of tokens and use a logarithmically-spaced learning rate grid from $2^{-16}$ to $2^{-5}$. 

\textbf{Gradient Conditioning (Fig.\ \ref{fig:gradient-conditioning}).} The effective rank of a matrix is defined as $\mathrm{erank}(\boldsymbol{\sigma}) = \exp (-\sum_{i=1}^{\min(m, n)} p_i \log p_i)$ where $p_i = \frac{\boldsymbol{\sigma}_i}{\| \boldsymbol{\sigma}\|_1}$ and $\boldsymbol{\sigma} \in \mathbb{R}^{\min(m, n)}$ contains the singular values of the gradient matrix. We normalize the effective rank by dividing it with $\min(m, n)$. For Muon, we show the gradient with respect to $\W$, and for \algname{} the gradient with respect to $\Rbf$. Each line shows one of the 12 layers and the bold line the average. 

\textbf{Gradient Noise (Fig.\ \ref{fig:scaled_cm_comparison}).} We base the gradient noise analysis on our 124M weight decay ablation setup and consider the configuration with the best final validation loss. This results in Muon ($\eta = 4 \times 10^{-3}, \lambda = 0.1$) and \algname{} without weight decay ($\eta = 6 \times 10^{-3}, \lambda = 0$) run. For each of these runs we save the model and optimizer states after $t = 1900, 3800, 5700, 7600$ and the final $T = 9538$ iterations. For every checkpoint, we first compute the true gradient $\nabla_\W \mathcal L \pare{\Wt}$ for Muon, and $\nabla_\g \Ft\pare{\gt, \Rt}, \nabla_\Rbf \Ft \pare{\gt, \Rt}$ for \algname{} checkpoints over the whole $5$B token training dataset $\{\xi_1, \dots, \xi_T\}$. In a second pass we then calculate, for each weight matrix $\W_t^{(i)}$, resp. $\g_t^{(i)}, \Rt^{(i)}$,
\begin{equation*}
    \sigma_{\W_t^{(i)}}^2 = \frac 1 T \sum_{\tau = 1}^{T}
    \trnorm{\nabla_\W \mathcal L \pare{\Wt} - \nabla_\W \mathcal L \pare{\Wt, \xi_\tau}}^2,
\end{equation*}
and similarly $\sigma_{\Rt^{(i)}}$ and $\sigma_{\gt^{(i)}}$. The values of $\zeta_\W, \zeta_\g$ and $\zeta_\Rbf$ are known analytically to be $\zeta_\W = \zeta_\Rbf = \sqrt{\min\{m, n\}}$ for $\W, \Rbf \in \mathbb R^{m \times n}$ and $\zeta_\g = \sqrt{m}$ for $\g \in \mathbb R^m$. We compute the noise term separately for each weight matrix, plot the median across weights as a bold line, and show the interquartile range as a shaded region. We find that the noise term of \algname{} remains consistently below that of Muon, with the gap widening over the course of training.

\textbf{nanoGPT Speedrunning (Fig.~\ref{fig:modded-nanogpt-learning-rate}). } For the nanoGPT speedrunning \citep{modded_nanogpt_2024} experiments, we leverage the Muon record available on github.\footnote{\url{https://github.com/KellerJordan/modded-nanogpt/blob/2c7027462732f2f7a76fb4b021066442c7b29298/records/track_1_short/2024-10-10_Muon/train_gpt2.py
}} 

\textbf{Multi-Seed Run for 500M (Table~\ref{tab:500M-multiseed}).\ } We adopt the 500M setup above and report mean and standard deviation across three seeds for each configuration. The grid is $\eta \in \{1, 2, 4\}\times 10^{-3}$ for all optimizers. \algname{} is run with $\lambda=0$; for Muon we report both $\lambda=0$ and $\lambda=0.1$, the latter identified as optimal in the 124M weight-decay ablation (Fig.~\ref{fig:weight-decay-ablation}). SOAP follows its hyperparameters from the 124M learning-rate ablation above with $\lambda=10^{-4}$.

\textbf{Resource Overhead (Table~\ref{tab:timing-overhead}).} We compare the overall step times for Muon and \algname{} by running 1K steps for the 500M model outlined above on 4 GH200 devices. Here, we report the median runtime for Muon not based on the PyTorch implementation, but based on a custom version that mirrors the implementation details of \algname{} as closely as possible. The reason for this deviation is that the PyTorch implementation does not provide a version which distributes the optimizer step across ranks, making a fair comparison in this setup infeasible.

\textbf{Qwen2-0.5B (Table~\ref{tab:qwen2-0.5B}).\ } We use the model configuration available on HuggingFace\footnote{\url{https://huggingface.co/Qwen/Qwen2-0.5B/blob/main/config.json}}, but lower the base frequency of RoPE to 10,000 to align the configuration with the initial phase of pre-training as outlined in \cite[Section~3.2]{yang_qwen2_2024}. We use the 10B split of FineWeb-Edu tokenized with the Qwen2-Tokenizer. We use a grid of $\eta \in \{10^{-3}, 2 \times 10^{-3}, 4 \times 10^{-3}, 8 \times 10^{-3}, 10^{-2}\}$ and tune $\lambda \in \{0, 0.1\}$ for Muon. 

\textbf{Magnitude Optimizer Ablation (Table~\ref{tab:magnitude-optimizer}).\ } We adopt the 500M setup and compare three choices for the magnitude optimizer in Algorithm~\ref{alg:algname} across $\eta \in \{1, 2, 4\}\times 10^{-3}$, all run with $\lambda=0$: \emph{Fixed} freezes $\g_t = \g_1$ at initialization, exposing the reparameterization without optimizing it; \emph{Signum} replaces Adam with signSGD with momentum, and \emph{Adam} is the default of Algorithm~\ref{alg:algname}.

\textbf{Licenses.} The FineWeb-Edu dataset \citep{lozhkov2024fineweb-edu} is released under the Open Data Commons Attribution License (ODC-By). \texttt{plainLM} \citep{ajroldi2024plainlm} and \texttt{modded-nanoGPT} \citep{modded_nanogpt_2024} are released under the MIT License. Qwen2 \citep{yang_qwen2_2024} is released under the Apache 2.0 License.

\end{document}